\documentclass[nohyperref]{article}

\usepackage[accepted]{icml2023}

\usepackage[utf8]{inputenc} %
\usepackage[T1]{fontenc}    %
\usepackage[obeyspaces]{xurl}
\usepackage[hidelinks,breaklinks]{hyperref}
\usepackage{booktabs}       %
\usepackage{amsfonts}       %
\usepackage{amsthm}
\usepackage{nicefrac}       %
\usepackage{microtype}      %
\usepackage{xcolor}         %
\usepackage{graphicx}       %
\usepackage{amsmath}        %
\usepackage{amssymb}        %
\usepackage{makecell}
\usepackage{multirow}
\usepackage{microtype}
\usepackage{enumitem}
\usepackage{tablefootnote}

\graphicspath{{art/}{subsumes/}{ood/}{plots/}}

\theoremstyle{definition}
\newtheorem{definition}{Definition}
\theoremstyle{plain}
\newtheorem{theorem}{Theorem}
\newtheorem{proposition}[theorem]{Proposition}

\def\R{\mathbb{R}}
\def\E{\mathbb{E}}
\def\w{\mathbf{w}}

\makeatletter
\def\paragraph{\@startsection{paragraph}{4}{\z@}{.1ex plus 0ex minus .1ex}{-1em}{\normalsize\bf}}
\makeatother
\def\synthdistill{\textsc{Distill}}  %
\def\synthcat{\textsc{Cat}}    %

\def\plotlabel#1{\textsf{\small``#1''}}

\def\eg.{\mbox{e.}\mbox{g.}}
\def\ie.{\mbox{i.}\mbox{e.}}

\icmltitlerunning{Learning useful representations for shifting tasks and distributions}

\begin{document}
\twocolumn[
    \icmltitle{Learning useful representations for shifting tasks and distributions}

\begin{icmlauthorlist}
\icmlauthor{Jianyu Zhang}{nyu}
\icmlauthor{L\'eon Bottou}{fbny,nyu}
\end{icmlauthorlist}

\icmlaffiliation{nyu}{New York University, New York, NY, USA.}
\icmlaffiliation{fbny}{Facebook AI Research, New York, NY, USA.}

\icmlcorrespondingauthor{Jianyu Zhang}{jianyu@nyu.edu}
    \icmlkeywords{Deep Learning, ICML}

    \vskip 0.3in
]

\printAffiliationsAndNotice{}  %

\begin{abstract}

Does the dominant approach to learn representations (as a side effect of optimizing an expected cost for a single training distribution) remain a good approach when we are dealing with multiple distributions? Our thesis is that \emph{such scenarios are better served by representations that are richer than those obtained with a single optimization episode.} 
We support this thesis with simple theoretical arguments and with experiments utilizing an apparently na\"{\i}ve ensembling technique: concatenating the representations obtained from multiple training episodes using the same data, model, algorithm, and hyper-parameters, but different random seeds. These independently trained networks perform similarly. Yet, in a number of scenarios involving new distributions, the concatenated representation performs substantially better than an equivalently sized network trained with a single training run. 
This proves that the representations constructed by multiple training episodes are in fact different. Although their concatenation carries little additional information about the training task under the training distribution, it becomes substantially more informative when tasks or distributions change. Meanwhile, a single training episode is unlikely to yield such a redundant representation because the optimization process has no reason to accumulate features that do not incrementally improve the training performance.

\end{abstract}

\section{Introduction}

Although the importance of features in machine learning systems was already clear when the Perceptron was invented \citep{rosenblatt-1957}, learning features from examples was often considered a hopeless task \citep{minsky-papert-1969}. Some researchers hoped that random features were good enough, as illustrated by the Perceptron. Other researchers preferred to manually design features using substantive knowledge of the problem \citep{simon-1989}.  This changed when \citet{rumelhart-1986} showed the possibility of feature learning as a side effect of the risk optimization. Despite reasonable concerns about the optimization of nonconvex cost functions, feature discovery through optimization has driven the success of deep learning methods.

There are however many cues suggesting that learning features no longer can be solely understood through the optimization of the expected error for a single data distribution. First, adversarial examples \citep{szegedy-2014} and shortcut learning \citep{geirhos-2020} illustrate the need to make learning systems that are robust to certain changes of the underlying data distribution and therefore involve multiple expected errors. Second, the practice of transferring features across related tasks \citep{tr-bottou-2011,collobert-2011,oquab-2014} is now viewed as foundational \citep{bommasani-2021} and intrinsically involves multiple data distributions and cost functions. It is therefore timely to question whether the optimization of a single cost function creates and accumulates features in ways that make the most sense in this broader context.

This contribution reports on experiments showing how the out-of-distribution performance of a deep learning model benefits from internal representations that are richer and more diverse than those computed with a single optimization episode. More precisely, \emph{although optimization can produce diverse features, a single run is unable to collect them all into a rich representation that performs better when tasks or distributions change}. In a time where many organizations deploy considerable resources training huge foundational models, this observation should be sobering.

\paragraph{Organization of the manuscript}
 
Section~\ref{sec:features} provides simple theoretical tools to discuss the value of features, discusses their consequences in-distribution and out-of-distribution, and approaches the notion of feature redundancy and feature richness. Sections~\ref{sec:supervisedtransfer}, \ref{sec:ssltransfer}, \ref{sec:metalearning}, and \ref{sec:oodlearning}, present experimental results pertaining respectively to supervised transfer learning, self-supervised transfer learning, meta-learning, and out-of-distribution learning.

\section{Related work}

\paragraph{Representations and optimization}

\citet{papyan-2020} shows how deep network representations collapse to a ``simplex equiangular tight frame'' when one trains for a very long time. \citet{shwartz2017opening} argue that the training process first develops representations in unsupervised mode, then prunes away the unnecessary features. Both papers associate this representation impoverishment with better generalization (in-distribution). We argue that it hurts performance when distributions change. Closer to our concerns, \citet{pezeshki2021gradient} describe the gradient starvation phenomenon which makes it hard to find the right features when a network already has spurious features. They do not however consider how to produce rich representations with multiple training episodes.

\paragraph{Ensembles}

\cite{dietterich2000ensemble} argues that model diversity is essential for the generalization performance of ensembles. \citet{ganaie2021ensemble} reviews deep ensembles with the same conclusion.  Our work focuses instead on scenarios involving multiple tasks or data distributions. We purposely refrain from engineering diversity, still observe substantial improvements, and draw conclusions about the undesirable properties of optimization.

\paragraph{Universal representations}

Several authors \citep{wang2022cross,dvornik2020selecting,bilen2017universal,gontijo2021no,li2021universal,li2022universal, chowdhury2021few} have recently proposed to collect features obtained with different tasks, datasets, network architectures, or hyper-parameters. The resulting so-``universal'' representations can be helpful for a variety of tasks. This approach is certainly interesting for practical problems but would not have allowed us to draw our conclusions.

\paragraph{Model soups}

Another line of work uses weight averaging to combine the properties 
of diverse networks~\citep{wortsman2022model,rame2022diverse}, with an increasing focus on leveraging models trained on multiple tasks to achieve a high performance on a task of interest~\citep{ilharco2022editing,rame2022recycling}. This engineered diversity provides for high performance, but does not allow the authors to draw conclusions about the optimization process itself. 

\paragraph{Shortcut learning and mitigation}

Several authors \citep[e.g.][]{huang2020self, teney2022evading} propose to work around the shortcut learning problem \citep{geirhos-2020} by shaping the last-layer classifier or introducing penalty terms in a manner that favors richer representations. \citet{zhang2022rich} argue that such additions make the optimization very challenging, but can be managed by initializing the networks with a rich representation constructed with an elaborate multi-step process. We show that rich representations can also be built by merely training the same network multiple times and combining their representations.

\section{Features and representations}
\label{sec:features}

This section provides a conceptual framework for talking about richness and diversity of representations. Although it seems natural to compare representations using information theory concepts such as mutual information, this approach is fraught with problems. For instance, the simplest way to maximize the mutual information $M(\Phi(x),y)$ between the representation $\Phi(x)$ and the desired output $y$ consists of making $\Phi$ equal to the identity. The information theoretic approach overlooks the main role of a feature extraction function, which is not filtering the information present in the inputs $x$, but formatting it in a manner exploitable by a simple learning system such as a linear classifier or a linear regression.\footnote{We choose linear classifiers as the ``simple learning system'' in our framework for the ease of theoretical analysis. This does not imply non-linear classifiers would behave differently. In fact, we empirically investigate another simple learning system, a cosine classifier, in the appendix Table \ref{tab:few_shot_synt_cat}.} The following framework relies on the linear probing error instead.

\paragraph{Framework}

We call \emph{feature} a function $x\,{\mapsto}\,\varphi(x)\,{\in}\,\R$, and we call \emph{representation} a set $\Phi$ of features. We use the notation $\w^\top\Phi(x)$ to denote the dot product $\sum_{\varphi\in\Phi} w_\varphi\,\varphi(x)$ where the coefficients $w_\varphi$ of vector $\w$ are indexed by the corresponding feature $\varphi$ and are assumed zero if $\varphi\notin\Phi$.

We assume for simplicity that our representations are exploited with a linear classifier trained with a convex loss~$\ell$. The expected loss of classifier $f$ is
\[ C_P(f) = \E_{(x,y)\sim P}\:\big[\,\ell(f(x),y)\,\big]\]
and the optimal cost achievable with representation $\Phi$ 
\begin{equation}
    \label{eq:ecost}
    C^*_P(\Phi) = \min_\w ~ C_P(f) ~~\text{with}~~ f: x\mapsto \w^\top \Phi(x)~.
\end{equation}
This construction ensures:
\begin{proposition}
\label{prop:min}
$C^*_P(\Phi_1 \cup \Phi_2) \leq C^*_P(\Phi_2)$ for all $\Phi_1$, $\Phi_2$.
\end{proposition}
Intuitively, if the combined representation $\Phi_1 \cup \Phi_2$ performs better than $\Phi_2$, then $\Phi_1$ must contain something useful that $\Phi_2$ does not. We formalize this using the word \emph{information} to actually mean \emph{linearly exploitable information about $y$}.

\begin{definition}
\label{def:more}
$\Phi_1$ contains information not present in $\Phi_2$\\
iff~
$C^*_P(\Phi_1\cup\Phi_2) < C^*_P(\Phi_2)$.
\end{definition}
Thanks to proposition~\ref{prop:min}, the opposite property becomes\,:
\begin{definition}
\label{def:notmore}
$\Phi_2$ contains all the information present in $\Phi_1$\\
iff~  $C^*_P(\Phi_1\cup\Phi_2) = C^*_P(\Phi_2)$.
\end{definition}
Finally we say that $\Phi_1$ and $\Phi_2$ carry equivalent information when $\Phi_2$ contains all the information present in $\Phi_1$, and $\Phi_1$ contains all the information present in $\Phi_2$\,:
\begin{definition}
\label{def:equiv}
$\Phi_1$ and $\Phi_2$ carry equivalent information\\ 
iff~ 
$C^*_P(\Phi_1) = C^*_P(\Phi_1 \cup \Phi_2) = C^*_P(\Phi_2)$.
\end{definition}
This definition is stronger\footnote{This is also weaker than using the quantity of information~$H$\,: writing $H(\Phi_1){=}H(\Phi_1{\cup}\Phi_2){=}H(\Phi_2)$ would imply that $\Phi_1$ and $\Phi_2$ are equal up to a bijection. Theorems~\ref{prop:ensemble} and~\ref{prop:opt} are important because this is not the case here.} than merely requiring equality $C^*_P(\Phi_1)\,{=}\,C^*_P(\Phi_2)$. In particular, we cannot improve the expected cost by constructing an ensemble\,:
\begin{theorem}
\label{prop:ensemble}
Let representations $\Phi_1$ and $\Phi_2$ carry equivalent information. Let $f_i(x){=}\w_i^{*\top}\Phi_i(x)$, for $i{\in}\{1,2\}$, be corresponding optimal classifiers. Then, for all $0{\leq}\lambda{\leq}1$\,,
\[  C^*_P(\lambda f_1 + (1-\lambda) f_2) = C^*_P(f_1) = C^*_P(f_2). \]
\end{theorem}
\begin{proof}
Let $\Phi=\Phi_1\cup\Phi_2$. Because the loss $\ell$ is assumed convex, the solutions of optimization problem \eqref{eq:ecost} form a convex set $S$. Since $C^*_P(\Phi_1){=}C^*_P(\Phi_1 \cup \Phi_2){=}C^*_P(\Phi_2)$, set $S$ contains $w^*_1$ and $w^*_2$, as well as any mixture thereof.
\end{proof}

 We now turn our attention to representations constructed by optimizing 
 both the representation $\Phi$ and the weights $\w$: 
 \begin{equation}
     \label{eq:optrepr}
     \min_\Phi C^*_P(\Phi)  = \min_{\Phi} \min_\w \E_{(x,y)\sim P} [ \ell( \w^\top \Phi(x), y) ]\,.
 \end{equation}
This idealized formulation optimizes the expected error without constraints on the nature and number of features. All its solutions problem carry equivalent information\,:
 \begin{theorem}
 \label{prop:opt}
 Let $\Phi_1$ and $\Phi_2$ be two solutions of problem~\eqref{eq:optrepr}.
 Then $\Phi_1$ and $\Phi_2$ carry equivalent information.
 \end{theorem}
\begin{proof}
Proposition~\ref{prop:min} implies $C^*_P(\Phi_1\cup\Phi_2) \leq C^*_P(\Phi_1)$.
Since $\Phi_1$ and $\Phi_2$ are both solutions of problem~\ref{eq:optrepr}, $C^*_P(\Phi_1) = C^*_P(\Phi_2) \leq C^*_P(\Phi_1\cup\Phi_2) \leq C^*_P(\Phi_1)$. 
\end{proof}

\paragraph{In-distribution viewpoint}

Consider a deep network that is sufficiently overparameterized to accommodate any useful representation in its penultimate layer. Assume that we are able to optimize its expected cost on the training distribution, that is, optimize its in-distribution generalization error.  Although repeated optimization episodes need not return exactly the same representations, Theorem~\ref{prop:opt} tells us that these representations \emph{carry equivalent information}; Definition~\ref{def:equiv} tells us that we cannot either improve the in-distribution test error by linear probing, that is, by training a linear layer on top of the concatenated representations; and Theorem~\ref{prop:ensemble} tells us that we cannot improve the test error with an ensemble of such networks. The performance of ensembles depends on the diversity of their components~\cite{dietterich2000ensemble,ganaie2021ensemble}, and nothing has been done here to obtain diverse networks.

In practice, we cannot truly optimize the expected error of an overparameterized network. The representations obtained with separate training episodes tend to carry equivalent information but will not do so exactly.\footnote{Experience shows however that repeated trainings on large tasks, such as \textsc{ImageNet}, yields networks with remarkably consistent training and testing performances.} Although an ensemble of such identically trained networks can still improve both the training and testing errors, using such similarly trained networks remains a poor way to construct ensembles when one can instead vary the training data, the hyper-parameters, or vary the model structure \cite{ganaie2021ensemble}. Engineering diversity escapes the setup of Theorem~\ref{prop:opt} because each component of the ensemble then solves a different problem. This is obviously better than relying on how the real world deviates from the asymptotic setup.

\paragraph{Out-of-distribution viewpoint}

Assume now that we train our network on a first data distribution $P(x,y)$, but plan to use these networks, or their representations, or their inner layers, with data that follow a different distribution $Q(x,y)$. Doing so also escapes the assumptions of our framework because the definition of representation carrying similar information (Definition~\ref{def:equiv}) critically depends on the data distribution. Representations that carry equivalent information for the training distribution $P$ need not carry equivalent information for a new distribution $Q$ at all.\footnote{Information theoretical concepts are also tied to the assumed data distribution. For instance, whether two features have mutual information critically depends on the assumed data distribution.}

Consider again representations obtained by performing multiple training episodes of the same network that only differ by their random seed.\footnote{The random seed here may determine the initial weights, the composition of the mini-batches, or the data augmentations. It does not affect the data distribution, the model structure, or even the training algorithm hyper-parameters.}  These representations roughly carry equivalent information with respect to the training distribution, but, at the same time, may be very far from carrying equivalent information with respect to a new distribution. 

If this is indeed the case, \emph{constructing an ensemble of such similarly trained networks can have a far greater effect on out-of-distribution data than in-distribution.} 
Experimental results reported in the following sections will demonstrate this effect. In fact, since we cannot know which of these representations or features might prove more informative on the new distribution, it seems wise to keep them all. \emph{Premature feature selection is not a smart way to prepare for distribution changes.}

\paragraph{Optimization dynamics}

There is growing evidence that implicit regularization in deep learning networks is  related to various flavors of sparsity (\eg. \citealp{andriushchenko2022sgd,blanc2020implicit}). In an oversimplified account of this complex literature, the learning process explores the feature space more or less randomly; features that carry incrementally useful information stick more than those who do not. Consider for instance a network with representation~$\Phi_t$ at iteration $t$ and a feature $\varphi\in\Phi_t$ whose information is already present in $\Phi_t\!{\smallsetminus}\{\varphi\}$ in the sense of Definition~\ref{def:notmore}. This feature does not incrementally improve the training distribution performance and therefore may not stick. Yet this feature might contain useful information when compared to a different representation, or when compared to $\Phi_t\!{\smallsetminus}\{\varphi\}$ under a different distribution.

Explicit regularization in deep networks, such as the ubiquitous slight weight decay, also tends to destroy features that appear redundant. \citet{papyan-2020} describes how representations collapse when one trains a network for a very long time. \citet{shwartz2017opening} describe competing processes that create representations and prune representations in all layers at once. 

\begin{table}[t]
    \par\vspace*{-2ex}
    \caption{Impact of L2 weight decay on supervised transfer learning between \textsc{Cifar10} and \textsc{Cifar100}.}
    \label{tab:supervise_transfer_cifar}
    \bigskip\centering
    \resizebox{0.39\textwidth}{!}{
    \begin{tabular}{c|cc}
    \toprule
     L2 weight decay    & $0$  & $5e-4$  \\
         \midrule
         \textsc{Cifar10} &
         91.41$\pm$0.81 & \textbf{94.89$\pm$0.23} \\
         \textsc{Cifar10$\rightarrow$Cifar100} &  
          \textbf{49.68$\pm$0.72} & 29.17$\pm$0.50 \\
        \midrule
         \textsc{Cifar100} &
         70.37$\pm$1.49 &\textbf{76.78$\pm$0.36} \\         
          \textsc{Cifar100$\rightarrow$Cifar10} &  
           \textbf{78.87$\pm$0.98} & 75.92$\pm$0.54 \\
    \bottomrule
    \end{tabular}
    }
\end{table}

Table~\ref{tab:supervise_transfer_cifar} reports on a simple experiment to illustrate how capacity control with regularization can help in-distribution performance but hurt when the distribution changes. We pre-train a \textsc{resnet18} on the \textsc{Cifar10} task and transfer its learned representation to a \textsc{Cifar100} task by linear probing (see setups in appendix \ref{apx:cifar10_100_sl}). %
Although the best in-distribution performance, 94.9\%, is achieved using a slight weight decay, the  representation learned \emph{without weight decay} transfers far better (49.7\% versus 29.2\%). The same observation holds when one reverses the role of the \textsc{Cifar10} and \textsc{Cifar100} datasets. 

\paragraph{Next steps}

The remaining sections of this paper describe experiments that investigate the effect of concatenating representations obtained by multiple training episodes that only differ by their random seed. 

Despite the \emph{intentional lack of diversity} of these ensembles, the performance improvements observed on tasks involving distribution changes are far greater than the in-distribution performance improvements. This proves that representations constructed by multiple training episodes are indeed different. Even though their concatenation carries little additional information for in-distribution, these experiments show how they become substantially more informative when tasks
or distributions change. 

Meanwhile, we obtain worse performance (a) when we train a network whose size matches that of the ensemble from scratch, or (b) when we fine-tune the concatenated representations in a single additional run. We contend that this happens because optimization inherently impoverishes the representations in a manner that makes sense in-distribution but hurts out-of-distribution, and we propose \emph{two-stage fine-tuning} (Figure~\ref{fig:twostageft}) to correct this behavior.

\section{Supervised transfer learning}
\label{sec:supervisedtransfer}

This section focuses on supervised transfer learning scenarios in which the representation learned using an auxiliary supervised task, such as the \textsc{ImageNet} object recognition task \citep{imagenet_cvpr09}, is then used for the target tasks, such as, for our purposes, the \textsc{Cifar10}, \textsc{Cifar100}, and \textsc{Inaturalist18} (\textsc{Inat18}) object recognition tasks \citep{krizhevsky2009learning,van2018inaturalist}. We distinguish the \emph{linear probing} scenario where the penultimate layer features of the pre-trained network are used as inputs for linear classifiers trained on the target tasks, and the \emph{fine tuning} scenario which uses back-propagation to further update the transferred features using the target task training data.\footnote{Code is available at \url{https://github.com/TjuJianyu/RRL}}

\begin{table}
    \par\vspace*{-2ex}
    \caption{Supervised transfer learning from \textsc{ImageNet} to \textsc{Inat18}, \textsc{Cifar100}, and \textsc{Cifar10} using linear probing. The \textsc{erm} {(empirical risk minimization)} rows provide baseline results. The \synthcat$n$ rows use the concatenated representations of $n$ separately trained networks. 
    }
    \label{tab:imagenet_sl_lineareval}
    \bigskip
    \centering
    
    \setlength\tabcolsep{1.2pt}
    \resizebox{0.488\textwidth}{!}{
    \begin{tabular}{c cc  |c| ccc }
    \toprule
                 &              &        &    {\small ID}             &      \multicolumn{3}{c}{\small Linear Probing (OOD)} \\
         method  & architecture     & params & \textsc{\small imagenet} &  \textsc{\small inat18}  & \textsc{\small cifar100} & \textsc{\small cifar10} \\
         \midrule
             \textsc{erm} & \textsc{resnet50}     & 23.5\textsc{m} & 75.58 & 37.91 & 73.23 & 90.57   \\
             \textsc{erm} & \textsc{resnet50w2}   & 93.9\textsc{m} & 77.58	& 37.34 & 72.65 & 90.86		\\
             \textsc{erm} & \textsc{resnet50w4}   & 375\textsc{m}  & 78.46	& 38.71 & 74.81	& 92.13	\\
            \midrule
             \textsc{erm} & 2$\times$\textsc{resnet50}    & 47\textsc{m}   & 75.03	& 39.34 & 74.36	& 90.94	\\
             \textsc{erm} & 4$\times$\textsc{resnet50}    & 94\textsc{m}   & 75.62	& 41.89 & 74.06	& 90.61	\\
             \midrule
             {\synthcat}2& 2$\times$\textsc{resnet50}     & 47\textsc{m}   & 77.57 & 43.26	& 76.10	& 91.86	\\
             {\synthcat}4& 4$\times$\textsc{resnet50}     & 94\textsc{m}   & 78.15 & 46.55	& 78.19	& 93.09	\\
            {\synthcat}10& 10$\times$\textsc{resnet50}    & 235\textsc{m}  & 78.36 & 49.65 	& 79.61	& 93.75	\\
        \bottomrule
    \end{tabular}
    }
\end{table}

\begin{figure*}[t]
    \centering
    \includegraphics[width=\textwidth]{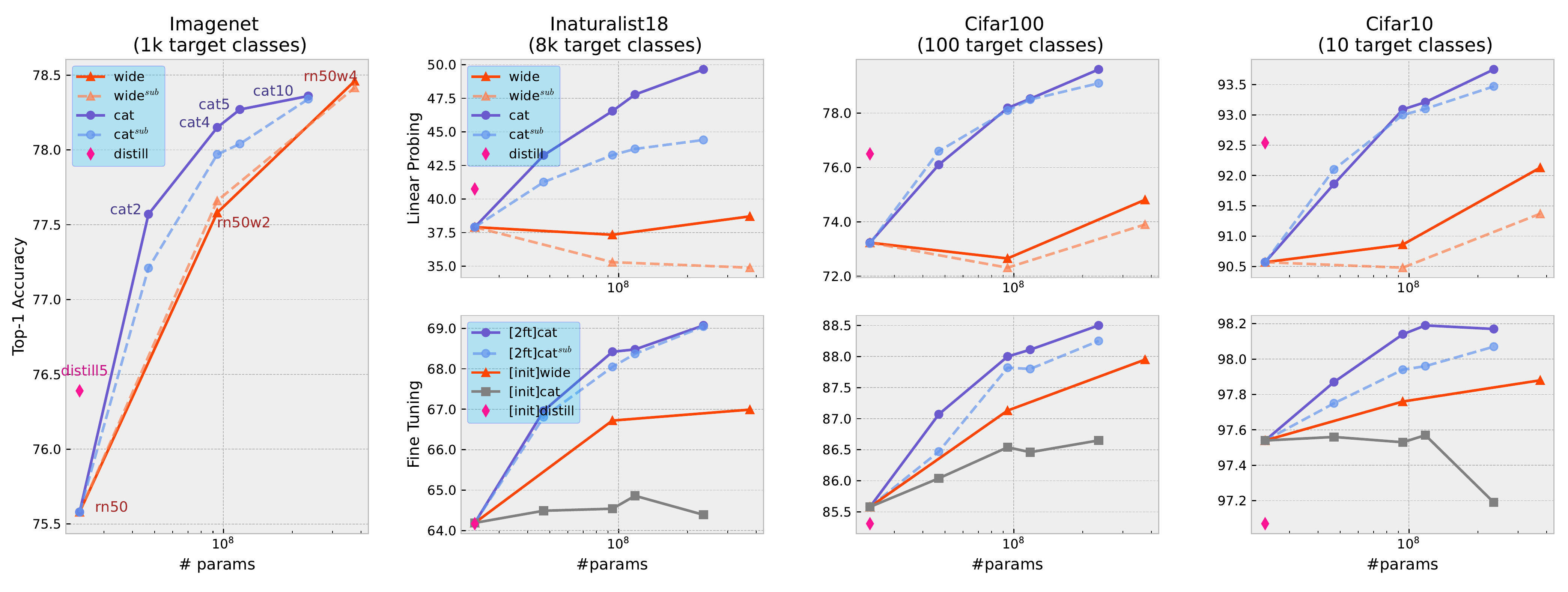}
    \caption{Supervised transfer learning from \textsc{ImageNet} to \textsc{Inat18}, \textsc{Cifar100}, and \textsc{Cifar10}.  The top row shows the superior linear probing performance of the \synthcat$n$ networks (blue, \plotlabel{cat}). The bottom row shows the performance of fine-tuned \synthcat$n$, which is poor with normal fine-tuning (gray, \plotlabel{[init]cat}) and excellent for two-stage fine tuning (blue, \plotlabel{[2ft]cat}). \synthdistill$n$ (pink, \plotlabel{distill}) representation is obtained by distilling \synthcat$n$ into one \textsc{resnet50} (we omit $\synthdistill$ in this section due to the space limit. see details in the appendix \ref{apx:imagenet_sl}). 
    }
    \label{fig:imagenet_sl_ft_2ft}
\end{figure*}
\paragraph{Linear probing}

The first three rows of Table~\ref{tab:imagenet_sl_lineareval}, labeled \textsc{erm}, provide baselines for the linear probing scenario, using respectively a \textsc{resnet50} network \citep{he-2016}, as well as larger variants \textsc{resnet50w}$n$ with $n$ times wider internal representations and roughly $n^2$ times more parameters.  The following two rows provide additional baseline results using networks \textsc{$n\times$resnet50} composed of respectively $n$ separate \textsc{resnet50} networks joined by concatenating their penultimate layers. Although these networks perform relatively poorly on the pre-training task \textsc{ImageNet}, their linear probing performance is substantially better than that of the ordinary \textsc{resnet}s.

The final three rows of Table~\ref{tab:imagenet_sl_lineareval}, labeled \synthcat$n$, are obtained by training $n$ separate \textsc{resnet50} networks on \textsc{ImageNet} with different random seeds, and using their concatenated representations as inputs for a linear classifier trained on the target tasks. This approach yields linear probing performances that substantially exceed that of comparably sized baseline networks. Remarkably, \synthcat$n$, with separately trained components, outperforms the architecturally similar \textsc{$n\times$resnet50} trained as a single network. See appendix \ref{apx:imagenet_sl} for experimental details.

These results are succinctly\footnote{In order to save space, all further results in the main text of this contribution are presented with such plots, with result tables provided in the appendix.} represented in the top row of Figure~\ref{fig:imagenet_sl_ft_2ft}. For each target task \textsc{Inat18}, \textsc{Cifar100}, and \textsc{Cifar10}, the solid curves show the linear probing performance of the baseline \textsc{resnet50w}$n$ (red, labeled \plotlabel{wide}) and of the \synthcat$n$ networks (blue, \plotlabel{cat}) as a function of the number of parameters of their inference architecture. 

The left plot (double height) of Figure~\ref{fig:imagenet_sl_ft_2ft} provides the same information in-distribution, that is, using the pre-training task as target task. In-distribution, the advantage of \synthcat$n$ vanishes when the networks become larger, possibly large enough to approach the conditions of Theorem~\ref{prop:opt}. The out-of-distribution curves (top row) are qualitatively different because they show improved performance all along.

An ensemble of $n$ \textsc{resnet50} networks is architecturally similar to the \synthcat$n$ models. Instead of training a linear classifier on the concatenated features, the ensemble averages $n$ classifiers
independently trained on top of each network. Whether this is beneficial depends on the nature of the target task and its training data (dashed blue, labeled \plotlabel{cat$^{\sf sub}$}). For completeness, we also present an ensemble baseline (dashed red plot, labeled \plotlabel{wide$^{\sf sub}$}) averaging $n$ linear classifiers trained on top of a random partition of the corresponding wide network features.

\paragraph{Fine-tuning} 

Having established, in the linear probing case, that transferring concatenated representations \synthcat$n$ outperforms transferring the representation of an equivalently sized network, we turn our attention to fine-tuning.

Fine-tuning is usually achieved by setting up a linear classifier on top of the transferred feature and training it on the target task data while allowing back-propagation to update the transferred features as well. The bottom row of Figure~\ref{fig:imagenet_sl_ft_2ft} shows the performance of this approach using the baseline network representations (red curve, 
 labeled \plotlabel{[init]wide}) and the concatenated representations (gray curve, labeled \plotlabel{[init]cat}), The latter perform very poorly.\footnote{The poor performance of plain fine-tuning had already been pointed out by \citet{kumar2022finetuning} and \citet{kirichenko2022last}.}

 \begin{figure}[t]
\centering
\includegraphics[height=0.36\linewidth]{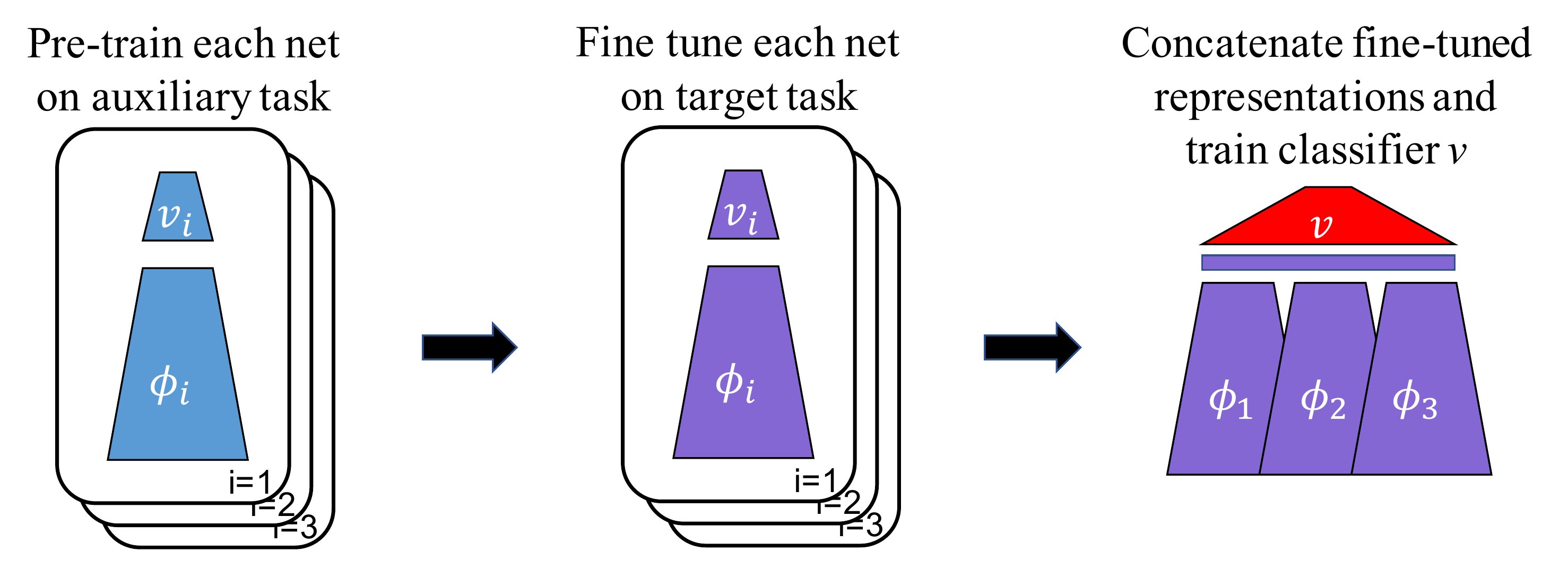}
\caption{\emph{Two-stage fine-tuning} consists of fine-tuning each network separately, then concatenating their feature extractors, now frozen, and training a final classifier. }
\label{fig:twostageft}
\end{figure}

We posit that fine-tuning with a single training episode impoverishes the initially rich representation. Instead, we propose \emph{two-stage fine tuning} which consists
of separately training $n$ networks on the pre-training task, separately fine-tuning them on the target task, and finally training a linear classifier on top of the concatenation of the $n$ separately fine-tuned representations (Figure~\ref{fig:twostageft}). The superior performance of two-stage fine-tuning is clear in the bottom row of Figure~\ref{fig:imagenet_sl_ft_2ft} (blue solid curve, labeled \plotlabel{{[2ft]}cat}). Ensembles of fine-tuned networks perform almost as well (blue dashed curve, labeled \plotlabel{[2ft]cat$^{\sf sub}$}).

The superior \textit{two-stage fune-tuning} performance, compared with the \textit{normal fine-tuning} (gray curve), may look counter-intuitive, since separately fine-tuning $n$ sub-networks is also likely to reduce the richness of the representation due to the in-distribution equivalence of information (Theorem \ref{prop:ensemble}). A similar phenomenon also exists in \textsc{ImageNet} pre-training in Table \ref{tab:imagenet_sl_lineareval}, where the ID (in-distribution) performance of \textsc{Cat}$n$ is substantially better than \textsc{erm} on the same $n\times$\textsc{resnet50} architectures. We believe that the difference is with the dynamics of the optimization process. In appendix \ref{apx:nresnet50_lags}, we show the accuracy of each leg of \textsc{erm} pretrained $n\times$\textsc{resnet50} are very disparate: one leg is doing all the work (The ID \textsc{imagenet} top-1 accuracy difference between legs is as large as 73\%). This is not the case in \textsc{Cat}$n$ pretraining.

\paragraph{Vision transformers}

Figure~\ref{fig:imagenet_vit_sl_tf} shows that transformer networks behave similarly. We carried out supervised transfer experiments using the original vision transformer, \textsc{ViT}, \citep{dosovitskiy2020image}, and using a more advanced version using carefully crafted data augmentations and regularization, \textsc{ViT(augreg)},  \citep{steiner2021train}. 
We use two transformers of two different sizes, {ViT-B/16} and {ViT-L/16}, pre-trained on \textsc{ImageNet21k}.\footnote{Checkpoints provided at \url{https://github.com/google-research/vision_transformer}.} Supervised transfer baselines (red, \plotlabel{wide\&deep} or \plotlabel{[init]wide\&deep}) are obtained by linear-probing and by fine-tuning on \textsc{ImageNet(1k)}. These baselines are outperformed by respectively linear-probing and \emph{two-stage fine tuning} on top of the concatenation of their final representations (\synthcat2). 

An even larger transformer architecture, {ViT-H/14}, yields about the same \textsc{ImageNet1k} fine-tuning performance as {ViT-L/16}, but lags 1\% behind \synthcat2, despite having twice as many parameters \cite{dosovitskiy2020image}. 
Experiments with two-stage fine-tuned \synthcat2 in \textsc{ViT(augreg)} show even better results, possibly because changing the random seed does not just changes the initial weights and the mini-batch composition, but also affects the data augmentations of the \textsc{ViT(augreg)} networks. 
\begin{figure}[h!]
    \centering
    \includegraphics[width=0.48\textwidth]{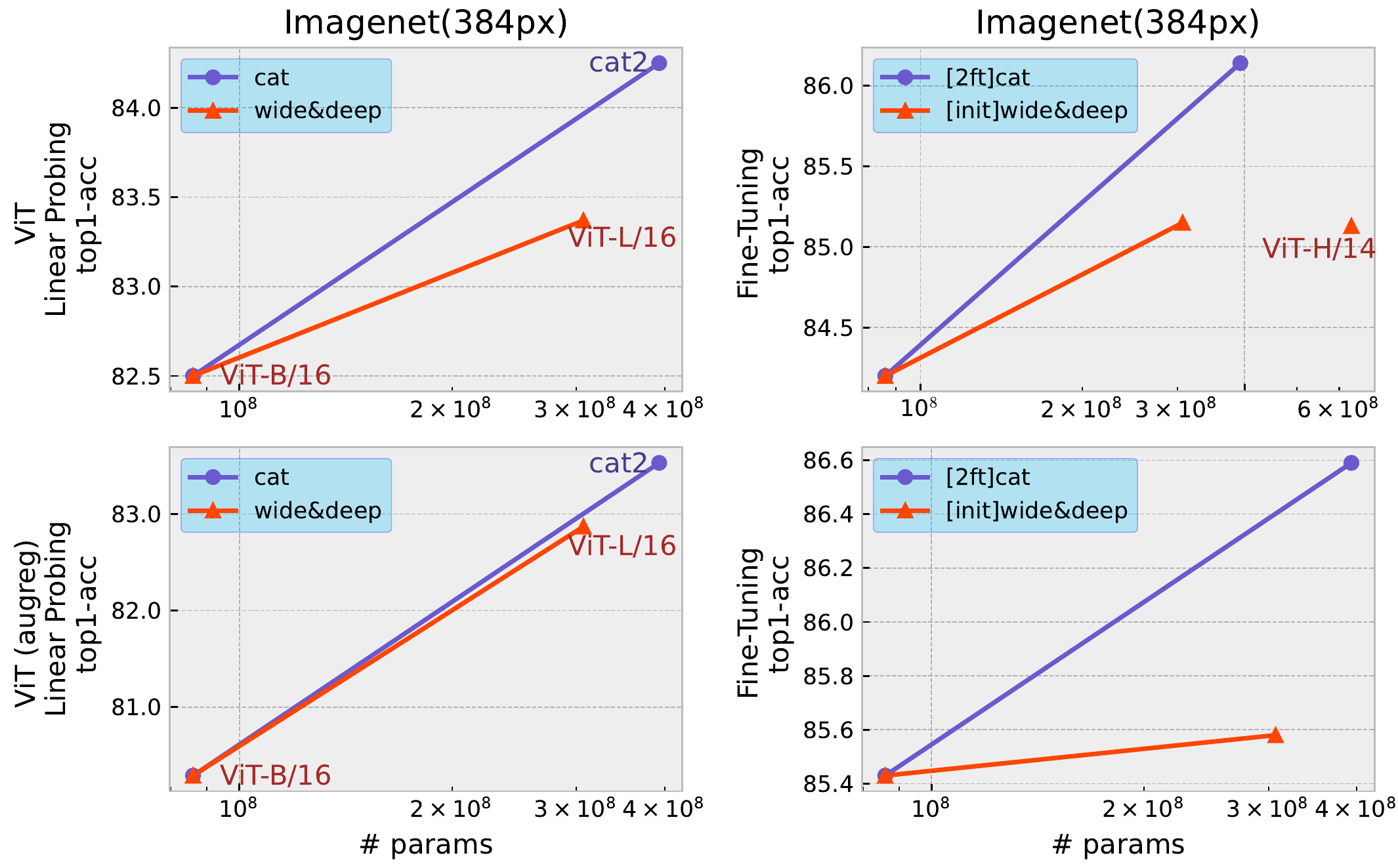}
    \caption{{Supervised transfer learning from \textsc{ImageNet21k} to \textsc{ImageNet} on vision transformers.}}
    \label{fig:imagenet_vit_sl_tf}
\end{figure}

\section{Self-supervised transfer learning}
\label{sec:ssltransfer}

\begin{figure*}[ht]
    \centering
    \includegraphics[height=0.395\textwidth]{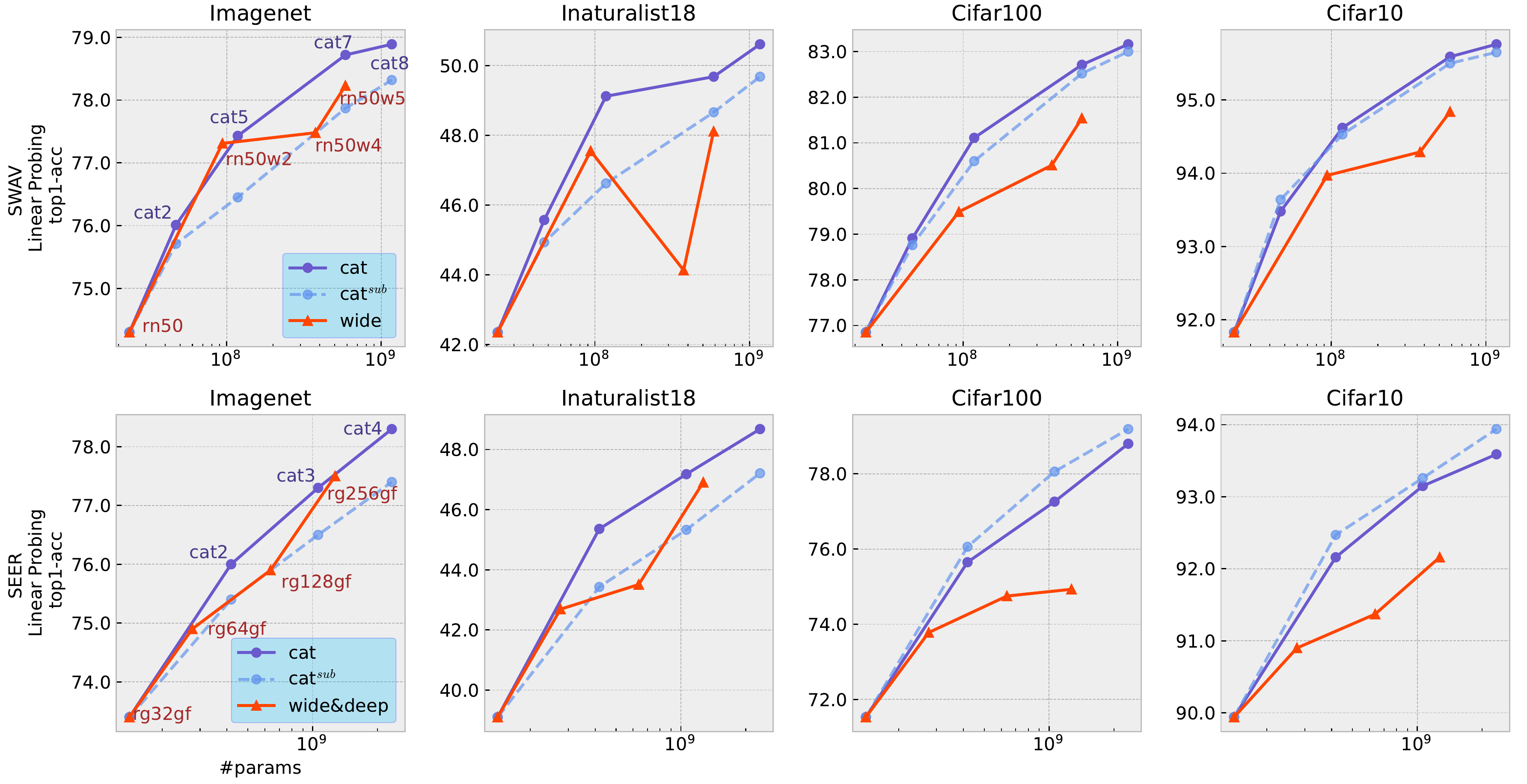}~~~~~
    \includegraphics[height=0.395\textwidth]{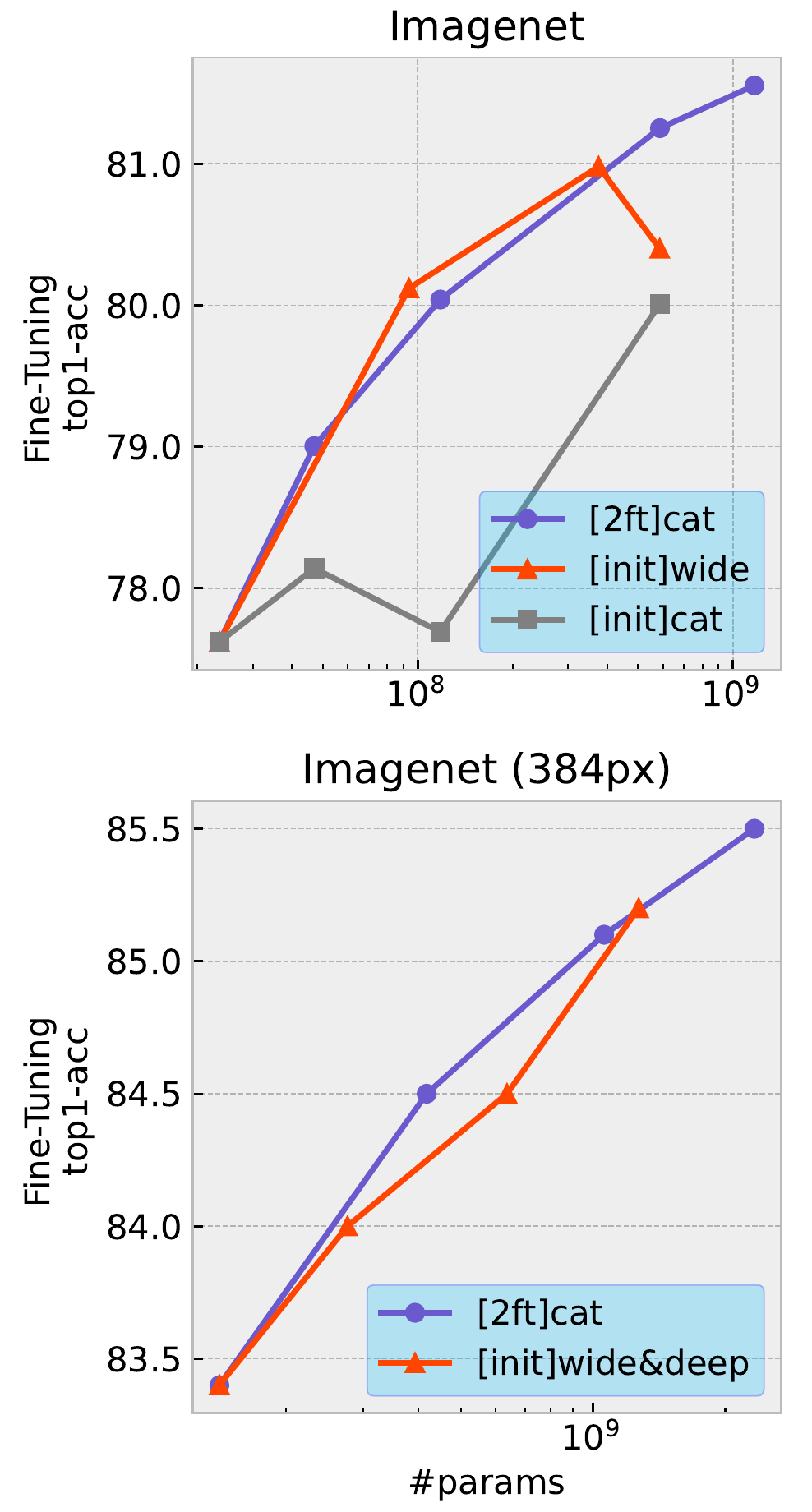}
    \caption{Self-supervised transfer learning with \textsc{swav} trained on unlabeled \textsc{ImageNet(1k)} (\emph{top row}) and with \textsc{seer} on \textsc{Instagram1B} (\emph{bottom row}). The constructed rich representation, \synthcat$n$, yields the best linear probing performance (\plotlabel{cat} and \plotlabel{cat$^{\sf sub}$}) for supervised \textsc{ImageNet}, \textsc{Inat18}, \textsc{Cifar100}, and \textsc{Cifar10} target tasks. The two-stage fine-tuning (\plotlabel{[2ft]cat}) matches equivalently sized baseline models (\plotlabel{[init]wide} and \plotlabel{[init]wide\&deep}), but with much easier training. The sub-networks of \synthcat5 (and \synthcat2) in \textsc{swav} hold the same architecture. Due to the space limitation, we put other fine-tuning curves in appendix \ref{apx:swav_additional_exp}.} 
    \label{fig:ssl_tf}
\end{figure*}

In self-supervised transfer learning (SSL), transferable representations are no longer constructed using a supervised auxiliary task, but using a training criterion that does not involve tedious manual labeling. We focus on schemes that rely on the knowledge of a set of acceptable pattern transformations. The training architecture then resembles a siamese network whose branches process different transformations of the same pattern. The SSL training objective must then balance two terms: on the one hand, the representations computed by each branch must be close or, at least, related; on the other hand, they should be prevented from collapsing partially \citep{jing2021understanding} or catastrophically \citep{chen2020simsiam}.  Although this second term tends to fill the representation with useful features, what is necessary to balance the SSL training objective might still exclude potentially useful features for the target tasks.

This section presents results obtained using \textsc{swav} pre-training using 1.2 million \textsc{ImageNet} images \citep{caron2020unsupervised} and using \textsc{seer} pre-training using 1 billion \textsc{Instagram1B} images \citep{goyal2022vision}. These experiments leverage the pre-trained models made available by the authors: five \textsc{resnet50} (four from our reproduction), one \textsc{resnet50w2}, one \textsc{resnet50w4} and one \textsc{resnet50w5} for the \textsc{swav} experiments;\footnote{\url{https://github.com/facebookresearch/swav}} one \textsc{regnet32gf}, one \textsc{regnet64gf}, one \textsc{regnet128gf}, and one \textsc{regnet256gf} (1.3B parameters) for the \textsc{seer} experiments.\footnote{\url{https://github.com/facebookresearch/vissl/tree/main/projects/SEER}}

The first four columns of Figure~\ref{fig:ssl_tf} present linear probing results for four target object recognition tasks: supervised \textsc{ImageNet}, \textsc{Inaturalist18}, \textsc{Cifar100}, and \textsc{Cifar10}. The baseline curves (red, labeled \plotlabel{wide} or \plotlabel{wide\&deep}) plot the performance of linear classifiers trained on top of the pre-trained SSL representations. The solid \synthcat$n$ curves were obtained by training a linear classifier on top of the concatenated representations of the $n$ smallest SSL pre-trained representations (solid blue, \plotlabel{cat}). The dash \synthcat$n$ curves train an ensemble of $n$ small classifiers on subsets of the concatenated representation (dash blue, \plotlabel{cat$^{\sf sub}$}).\footnote{Likewise the supervised transfer learning experiments, each small classifier learns on the representation of a sub-network (e.g. \textsc{regnet32gf}, \textsc{regnet64gf}). Now the representation subset cannot be treated as random subsets of the concatenated representation anymore, because the model architectures are not always the same. So we omit the ensemble classifiers for red curves.} Overall, the \synthcat$n$ approach offers the best performance.

The last column of Figure~\ref{fig:ssl_tf} presents results with fine-tuning for the supervised \textsc{ImageNet} task. Our \emph{two-stage fine-tuning} approach (as Figure~\ref{fig:twostageft}) matches the performance of equivalently sized baseline networks. In particular, the largest \synthcat4 model using \textsc{seer} pre-training, with 2.3B parameters, achieves 85.5\% correct classification rate, approaching the 85.8\% rate of the largest baseline network in \textsc{seer} \citep{goyal2022vision}, \textsc{regnet10B} with 10B parameters. Of course, separately training and fine-tuning the components of the $\synthcat$4 network is far easier than training a single \textsc{regnet10B} network.

Additional results using \textsc{SimSiam} \citep{chen2020simple} and with distillation are provided in appendix \ref{apx:simsiam_cifar}. Other experiment details are provided in appendix \ref{apx:ssl}.

\section{Meta-learning \& few-shots  learning}
\label{sec:metalearning}

Each target task in the few-shots learning scenario comes with only a few training examples. One must then consider a large collection of target tasks to obtain statistically meaningful results.

We follow the setup of \citet{closelookatfewshot}\footnote{We are aware of various existing few-shot benchmarks, such as MetaDataset \citep{triantafillou2019meta}, that contain more datasets than \citet{chen2020simple}. We choose \citet{chen2020simple}, because it is enough to validate our ideas in section \ref{sec:features}.
} in which the base task is an image classification task with a substantial number of classes and examples per class, and the target tasks are five-way classification problems involving novel classes that are distinct from the base classes and come with only a few examples. Such a problem is often cast as a \emph{meta learning} problem in which the base data is used to learn how to solve a classification problem with only a few examples. \citet{closelookatfewshot} find that excellent performance can be achieved using simple baseline algorithms such as supervised transfer learning with linear probing (\textsc{Baseline}) or with a cosine-based final classifier (\textsc{Baseline++}). These baselines match and sometimes exceed the performance of common few-shots algorithms such as \textsc{maml} \citep{finn2017model}, \textsc{RelationNet} \citep{sung2018learning}, \textsc{MatchingNet} \citep{matchingnet}, and \textsc{ProtoNet} \citep{protonet}.

\begin{figure}[h]
\centering
    \includegraphics[width=0.4\textwidth]{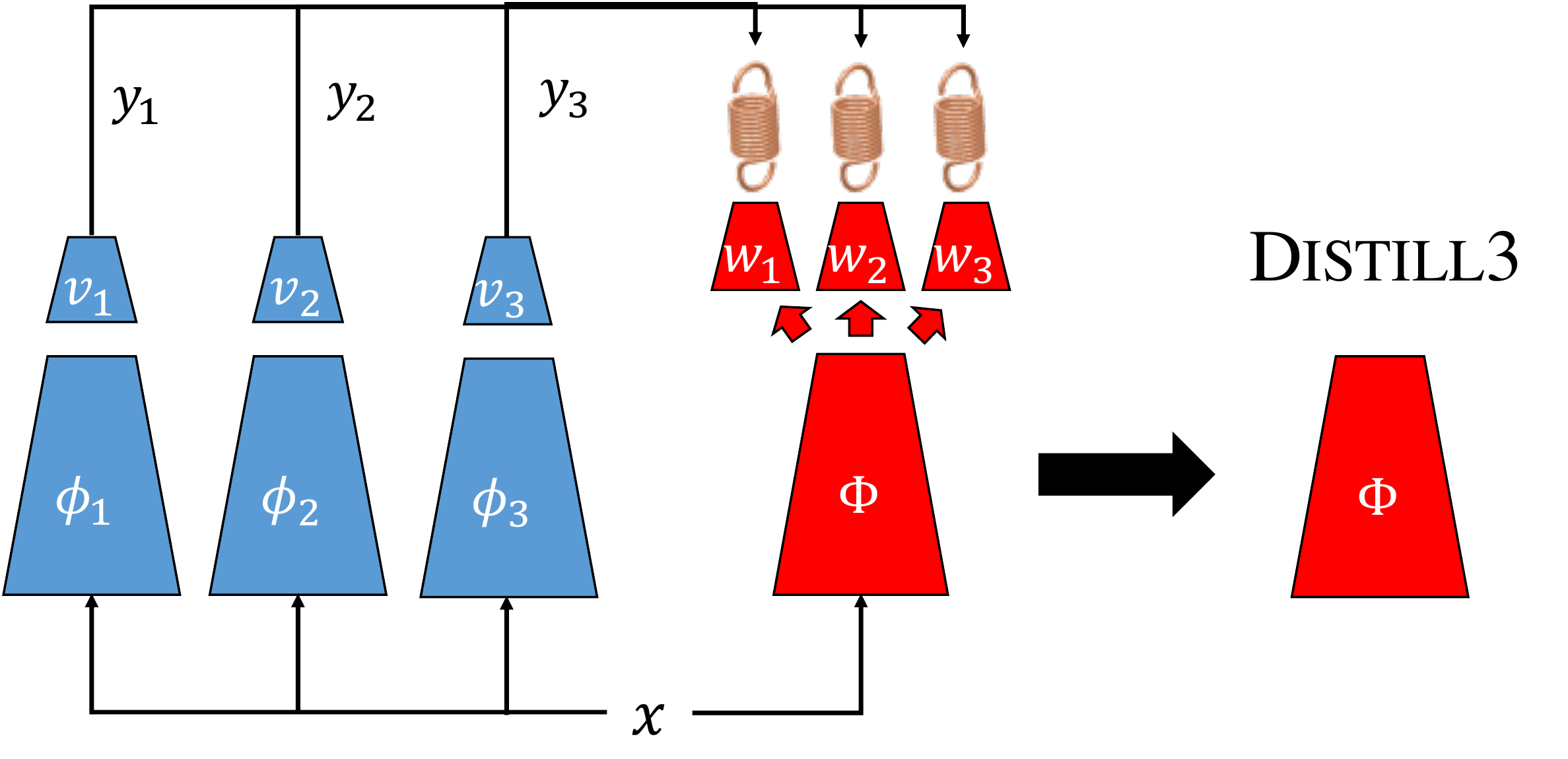}
    \caption{(\synthdistill$n$) A multiple head network (red) trained to predict the outputs of the pre-trained networks  $\Phi_1,\Phi_2,\cdots$ (blue) must develop a representation $\Phi$ that subsumes those of all the blue networks. The same distillation process is used by the \textsc{Bonsai} algorithm \citep{zhang2022rich} but after training the networks with adversarially re-weighted data. }
    \label{fig:distill} 
\end{figure}

\begin{figure}[t]
    \centering
    \includegraphics[width=0.485\textwidth]{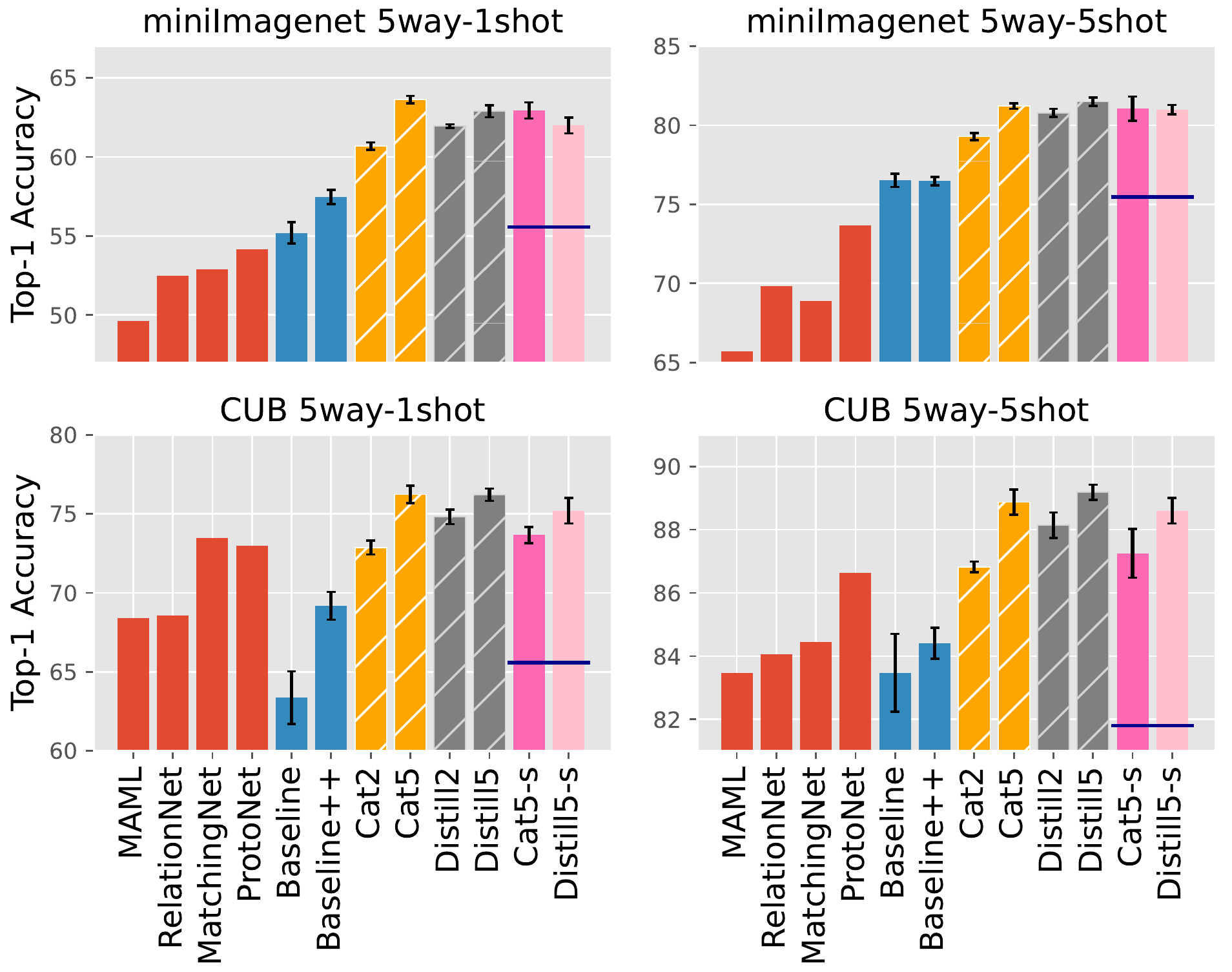}
    \caption{Few-shot learning performance on \textsc{MiniImageNet} and \textsc{Cub}. Four common few-shot learning algorithms are shown in red (results from \citet{closelookatfewshot}). Two supervised transfer methods, with either a linear classifier (\textsc{Baseline}) or cosine-based classifier (\textsc{Baseline++}) are shown in blue. The $\synthdistill$ and $\synthcat$ results, with a cosine-base classifier, are respectively shown in orange and gray. The \synthcat5\textsc{-s} and \synthdistill5\textsc{-s} results were obtained using five snapshots taken during a single training episode with a relatively high step size. The dark blue line shows the best individual snapshot. Standard deviations over five repeats are reported.}
    \label{fig:few_shots}
\end{figure}

Figure~\ref{fig:few_shots} reports results obtained with a \textsc{resnet18} architecture on both the \textsc{MiniImageNet} \citep{matchingnet} and \textsc{Cub} \citep{wah2011caltech} five ways classification tasks with either one or five examples per class as set up by \citet{closelookatfewshot}. The \textsc{maml}, \textsc{RelationNet}, \textsc{MatchingNet}, and \textsc{ProtoNet} results (red bars) are copied verbatim from \citep[table A5]{closelookatfewshot}. The \textsc{Baseline} and \textsc{Baseline++} results were further improved by a systematic L2 weight decay search procedure (see appendix \ref{apx:meta_baseline_pretrain}). All these results show substantial variations across runs, about $4\%$ for \textsc{Cub} and $2\%$ for \textsc{MiniImageNet}.

The \synthcat$n$ and \synthdistill$n$ results were then obtained by first training $n$ \textsc{resnet18} on the base data with different seeds, constructing a combined (rich) representation by either concatenation or  distillation (as Figure~\ref{fig:distill}), then, for each task, training a cosine distance classifier using the representation as input. Despite the high replication variance of the competing results, both $\synthdistill$ and $\synthcat$ show very strong performance. Note that naively increasing model architecture, e.g. from \textsc{resnet18} to \textsc{resnet34}, can only gain limited improvements ($\leq 1\%$, \citet{chen2020simple}, table A5) and is still lagging behind $\synthcat$ and $\synthdistill$.

The pink bars (\synthcat5\textsc{-s} and \synthdistill5\textsc{-s}) in Figure~\ref{fig:few_shots}, concatenate or distill five snapshots taken at regular intervals during a single training episode with a relatively high step size (0.8), achieve a similar few-shots learning performance as \synthcat5 and \synthdistill5, perform substantially better than the best individual snapshot (dark blue line). \emph{It implies that diverse features are discovered and then abandoned but not accumulated during the optimization process.} More results and details, as well as a comparison with conditional meta-learning algorithms \cite{wang2020structured, denevi2022conditional, rusu2018meta}, are shown in appendix \ref{apx:meta-learning}.

\section{Out-of-distribution generalization}
\label{sec:oodlearning}

In the out-of-distribution generalization scenario, we seek a model that performs well on a family of data distributions, also called environments, on the basis of a finite number of training sets distributed according to some of these distributions. \citet{irm} propose an invariance principle to solve such problems and propose the \textsc{IRMv1} algorithm which searches for a good predictor whose final linear layer is simultaneously optimal for all training distributions.
Since then, a number of algorithms exploiting similar ideas have been proposed, such as \textsc{vREx} \citep{vrex}, \textsc{Fishr} \citep{rame2021fishr}, or \textsc{CLOvE} \citep{wald2021calibration}. Theoretical connections have been made with multi-calibration \citep{hebertjohnson18a,wah2011caltech}. Alas, the performance of these algorithms remains wanting \citep{gulrajani2021in}.
\citet{zhang2022rich} attribute this poor performance to the numerical difficulty of optimizing the complicated objective associated with these algorithms. They propose to work around these optimization problems by providing initial weights that already extract a rich palette of potentially interesting features constructed using the \textsc{Bonsai} \citep{zhang2022rich}  algorithm.

Following \citet{zhang2022rich}, we use the \textsc{Camelyon17} tumor classification dataset \citep{bandi2018detection} which contains medical images collected from five hospitals with potentially different devices and procedures. As suggested in \citet{koh2021wilds}, we use the first three hospitals as training environments and the fifth hospital for testing. \textsc{ood}-tuned results are obtained by using the fourth hospital to tune the various hyper-parameters. \textsc{iid}-tuned results only use the training distributions (see details in appendix \ref{apx:ood}). The purpose of our experiments is to investigate whether initializing with the {\synthdistill} or {\synthcat} algorithm provides a computationally attractive alternative to \textsc{Bonsai}.

\begin{table}[t]
\par\vspace*{-2ex}
    \caption{Test accuracy on the \textsc{Camelyon17} dataset with \textsc{DenseNet121}. We compare various initialization (\textsc{ERM}, \synthcat$n$, \synthdistill$n$, and \textsc{Bonsai}) for two algorithms \textsc{vREx} and \textsc{ERM} using either the \textsc{iid} or \textsc{ood} hyperparameter tuning method. The standard deviations over 5 runs are reported.    }
    \par\vspace*{-2ex}
    \label{tab:camelyon17_synt_cat}
    \centering
    \bigskip
    \resizebox{0.485\textwidth}{!}{
    \begin{tabular}{c | cc |cc}
    \toprule
    &      \multicolumn{2}{c|}{\small IID-Tune}      & \multicolumn{2}{c}{\small OOD-Tune} \\
     & {\small \textsc{vREx}} & {\small \textsc{ERM}} & {\small \textsc{vREx}} & {\small \textsc{ERM}} \\
                   \midrule
{\small \textsc{ERM}} & 69.6$\pm$10.5   &   66.6$\pm$9.8   &   70.6$\pm$10.0  &   70.2$\pm$8.7 \\
\midrule
{\synthcat}2 & 74.3$\pm$8.0   &   74.3$\pm$8.0   &   73.7$\pm$8.1   &   74.2$\pm$8.1 \\
{\synthcat}5 & 75.2$\pm$2.9   &   75.0$\pm$2.7   &   74.9$\pm$3.3   &   75.1$\pm$2.8 \\
{\synthcat}20 & 76.4$\pm$0.5   &   76.5$\pm$0.5   &   76.8$\pm$0.9   &   76.4$\pm$0.9 \\
\midrule
{\synthdistill}2 & 67.1$\pm$4.7   &   66.9$\pm$4.8   &   67.4$\pm$4.3   &   66.7$\pm$4.2 \\
{\synthdistill}5 & 69.9$\pm$7.4   &   69.9$\pm$6.9   &   71.8$\pm$5.0   &   69.9$\pm$6.3 \\

{\synthdistill}20 & 73.3$\pm$2.5   &   73.2$\pm$2.3   &   74.8$\pm$3.2   &   73.1$\pm$2.7 \\

   \midrule
    \textsc{Bonsai}2\tablefootnote{We apply \textsc{Bonsai} algorithm with 2 discovery episodes. Check \citet{zhang2022rich} for more details.} &  77.9$\pm$2.7 &  78.2$\pm$2.6 & 79.5$\pm$2.7 & 78.6$\pm$2.6 \\ 
    \bottomrule
    \end{tabular}
    }
    \par\vspace*{-2ex}
\end{table}

Table~\ref{tab:camelyon17_synt_cat} compares the test performance achieved by two algorithms, \textsc{vREx} and \textsc{ERM}, after initializing with \textsc{ERM}, \synthcat$n$, \synthdistill$n$, and \textsc{Bonsai}, in both the \textsc{iid}-tune and \textsc{ood}-tune scenarios.  The {\synthcat} and {\synthdistill} initialization perform better than \textsc{ERM} but not as well as \textsc{Bonsai}. \emph{This result clearly shows the need to research better ways to train networks in a manner that yields diverse representations.} Although this contribution shows that simply changing the seed (as in {\synthcat} and {\synthdistill}) can achieve good results, the experience of deep ensembles \citep{gontijo-lopes2022no} suggests that more refined diversification methods might yield substantially better representations.

\section{Conclusion}

Using a simple theoretical framework and a broad range of experiments, we show that deep learning scenarios that involve changing tasks or distributions are \emph{better served by representations that are richer than those obtained with a single optimization episode.} In a time where many organizations
deploy considerable resources training huge foundational
models, this conclusion should be sobering.

The simple multiple-training-episode approach $\synthcat$ constructs such richer representation with excellent performances in various scenarios. The \emph{two-stage fine tuning} method works around the poor performance of normal fine-tuning in various transfer scenarios.

More importantly, this work provides a lot of room for new representation learning algorithms that move away from relying solely on a single optimization episode.

\section*{Acknowlegments}
The authors acknowledge stimulating discussions with Alexandre Ram\'e, Diane Bouchacourt and David Lopez-Paz. The authors also acknowledge support from the National Science Foundation (NSF Award 1922658) and from the Canadian
Institute for Advanced Research (CIFAR).

\raggedbottom
\bibliography{main.bib}
\bibliographystyle{plainnat}

\clearpage
\onecolumn
\appendix

\begingroup
\begin{center}
    \par\vspace*{-1ex}
    \Large \textbf{Supplementary Material}
\end{center}
\par\vspace{2ex}
\endgroup

\section{\textsc{Cifar} supervised transfer learning}
\label{apx:cifar10_100_sl}
\textsc{Cifar10} supervised transfer learning experiments train a \textsc{resnet18} network on the \textsc{Cifar10} dataset with/without L2 weight decay (4e-5) for $200$ epochs. During training, we use a SGD optimizer \citep{bottou2018optimization} with momentum=0.9, initial learning rate=0.1, cosine learning rate decay, and batch size=128. As to data augmentation, we use \textsc{RandomResizedCrop} (crop scale in $[0.8, 1.0]$), aspect ratio in $[3/4, 4/3]$) and \textsc{RandomHorizontalFlip}. During testing, the input images are resized to $36\times36$ by bicubic interpolation and \textsc{CenterCroped} to $32\times32$. All input images are normalized by $mean=(0.4914, 0.4822, 0.4465), std=(0.2023, 0.1994, 0.2010)$ at the end. 

Then transfer the learned representation to \textsc{Cifar100} dataset by training a last-layer linear classifier (linear probing). The linear layer weights are initialized by Gaussian distribution $\mathcal{N}(0, 0.01)$. The linear probing process shares the same training hyper-parameters as the supervised training part except for a zero L2 weight decay in all cases. 

The \textsc{Cifar100} supervised transfer learning experiments swap the order of \textsc{Cifar100} and \textsc{Cifar10}.

\section{\textsc{ImageNet} supervised transfer learning}
\label{apx:imagenet_sl}
\subsection{Experiment settings}
\label{apx:imagenet_sl_settings}

{\paragraph{Image Preprocessing:} Following \citet{he2016deep}, we use \textsc{RandomHorizontalFlip} and \textsc{RandomResizedCrop} augmentations for all training tasks. For \textsc{ImageNet} and \textsc{Inat18}, the input images are normalized by $mean=(0.485, 0.456, 0.406), std=(0.229, 0.224, 0.225)$. For \textsc{Cifar}, we use the same setting as Appendix \ref{apx:cifar10_100_sl}.}

\paragraph{\textsc{Imagenet} Pretraining:} The \textsc{resnet}s are pre-trained on \textsc{ImageNet} with the popular protocol of \citet{goyal2017accurate}: a SGD optimizer with momentum=0.9, initial learning rate=0.1, batch size=256, L2 weight decay=1e-4, and 90 training epochs. The learning rate is multiplied by 0.1 every 30 epochs. By default, the optimizer in all experiments is SGD with momentum=0.9.

\paragraph{\textsc{Distill}:} To distill the {\synthcat}$n$ representations $[\phi_1, \dots \phi_n]$ ($n\times$\textsc{resnet50}) into a smaller representation $\Phi$ (\textsc{resnet50}), we use the multi-head architecture as Figure \ref{fig:distill}. Inspired by \citet{hinton2015distilling}, we use the Kullback–Leibler divergence loss to learn $\Phi$ as:
\begin{equation}
\label{eq:distill}
    \min_{\Phi, w_0, \dots, w_n}\sum_{i=0}^n\sum_x\bigg[ \tau^2\mathcal{L}_{kl}\Big(s_{\tau}\big({v_i} \circ {\phi_i}(x)\big) ~||~ w_i \circ \Phi(x) \Big)\bigg],
\end{equation}

where $s_{\tau}(v)_i = \frac{e^{v_i/\tau}}{\sum_k{e^{v_k/\tau}}}$ is a softmax function with temperature $\tau$, $v_i$ is the learned last-layer classifier of $i^{th}$ sub-network of \synthcat$n$.

In the {\synthdistill} experiments, we distill five separately trained \textsc{resnet50} into one \textsc{resnet50} according to Eq \ref{eq:distill} with $\tau=10$. We use a SGD optimizer with momentum=0.9, batch size=2048, and weight decay=0. The initial learning rate is 0.1 and warms up to 0.8 within the first 5 epochs. Then learning rate decays to 0.16 and 0.032 at $210^{th}$ and $240^{th}$ epochs, respectively. The total training epochs is 270.

\paragraph{Linear probing:}
\begin{itemize}
    \item \textbf{\textsc{ImageNet}:} The \textsc{ImageNet} linear probing experiments train a linear classifier with the same hyper-parameters as \textsc{ImageNet} pretraining. By default, the last linear classifier in all linear probing experiments is initialized by $\mathcal{N}(0, 0.01)$.
    \item \textbf{\textsc{Inat18}, \textsc{Cifar100}, \textsc{Cifar10}:} Following the settings of \citet{goyal2022vision}, the linear probing experiments (on \textsc{Inat18}, \textsc{Cifar100}, \textsc{Cifar10}) adds a \textsc{BatchNorm} layer before the linear classifier to reduce the hyper-parameter tuning difficulty. The learning rate is initialized to 0.01 and multiplied by 0.1 every 8 epochs. Then train these linear probing tasks for 28 epochs by SGD Nesterov optimizer with momentum=0.9, batch size 256. Note that \textsc{BatchNorm} + a linear classifier is still a linear classifier during inference. We tune L2 weight decay from \{1e-4, 5e-4, 1e-3, 5e-3, 1e-2, 5e-2\} for \textsc{Cifar100} and \textsc{Cifar10}, \{1e-6, 1e-5, 1e-4\} for \textsc{Inat18}.

\end{itemize}

\paragraph{Fine-tuning:} As to the fine-tuning experiments (on \textsc{Cifar100}, \textsc{Cifar10}, and \textsc{Inat18}), we tune the initial learning rate from \{0.005, 0.01, 0.05\}, training epochs from \{50, 100\}. We further tune L2 weight decay from \{0, 1e-5, 1e-4, 5e-4\} for \textsc{Cifar100} and \textsc{Cifar10}, \{1e-6, 1e-5, 1e-4\} for \textsc{Inat18}. A cosine learning rate scheduler is used in fine-tuning experiments. A 0.01 learning rate and 100 training epochs usually provide the best performance for these three datasets. So we fix these two hyperparameters in the following supervised learning two-stage fine-tuning experiments and self-supervised learning experiments.

\paragraph{Two-stage fine-tuning:} For the two-stage fine-tuning experiments, we separately fine-tune each sub-network (i.e. \textsc{resnet50}) of the \textsc{Cat}$n$ architecture by the same protocol as the normal fine-tuning above. Then train a last-layer linear classifier on top of the concatenated fine-tuned representation. The last-layer linear classifier training can be very efficient with a proper weights initialization strategy. In this work, we initialize the last-layer classifier $w$ (including the bias term) by concatenating the last-layer classifier of each fine-tuned sub-network $w_i$, $w \leftarrow {[w_0^\top, \dots, w_n^\top]^\top}/{n}$. Then we only need to train the last-layer classifier $w$ for 1 epoch with a learning rate = $1e-3$ for \textsc{Cifar} and $1e-5$ for \textsc{Inat18}.

\subsection{Performance difference between legs (subnetworks) in \textsc{erm} pretrained $n\times$\textsc{resnet50}}
\label{apx:nresnet50_lags}
Table \ref{tab:subnetwork_acc} showcases the performance difference between legs of \textsc{erm} pretrained $n\times$\textsc{resnet50}. In the $n\times$\textsc{resnet50}, one leg is doing all the work. In the \textsc{Cat}$n$ pretrained $n\times$\textsc{resnet50}, this is not the case. We believe the difference comes from optimization dynamics.
\begin{table*}[h]
    \centering
     \caption{Top-1 \textsc{ImageNet} accuracy of each leg 
     (\textsc{resnet50}) of \textsc{erm} pre-trained $n$\textsc{resnet50}. To solely showcase the difference between the representation of legs, we report the training accuracy of fitting a linear classifier on top of the penultimate layer representation of each leg (subnetwork). }
     \bigskip
    \begin{tabular}{c|cccc}
    \toprule
    & subnetwork0 & subnetwork1 & subnetwork2 & subnetwork3 \\
    \midrule
2$\times$\textsc{resnet50} & 73.94 & 18.05 & - & - \\
4$\times$\textsc{resnet50} & 9.25 & 74.33 & 0.40 & 0.96 \\
\bottomrule
    \end{tabular}
    \label{tab:subnetwork_acc}
\end{table*}

\subsection{Experiments on a deeper architecture: \textsc{resnet152}}
\label{apx:imagenet_sl_resnet152}

Similar to Table  \ref{tab:imagenet_sl_lineareval} in section \ref{sec:supervisedtransfer}, Table \ref{tab:imagenet_sl_rn152_lineareval} provides similar experiments on a deeper architecture \textsc{resnet152}. {\synthcat}$n$ exceeds ERM on \textsc{ImageNet}, \textsc{Cifar10}, \textsc{Cifar100} , and \textsc{Inat18} linear probing tasks. 
\begin{table}[ht]
    \caption{Imagenet supervised transfer learning performance on a deep architecture \textsc{resnet152}.}
    \bigskip
    \label{tab:imagenet_sl_rn152_lineareval}
    \centering
   
    \begin{tabular}{cc |c| ccc }
    \toprule
                 &               &  ID      & \multicolumn{3}{c}{Linear Probing (OOD)} \\ 
         method  & architecture  & \textsc{ImageNet} & \textsc{Cifar10} & \textsc{Cifar100} & \textsc{Inat18}  \\
         \midrule
             ERM & \textsc{resnet152}     & 77.89 &	92.50 &	76.23 & 39.70 \\
        \midrule
             {\synthcat}2& $2\times$\textsc{resnet152}      & 79.34 &	94.26 &	79.15 & 45.42\\
             {\synthcat}5& $5\times$\textsc{resnet152}      & 80.14 &	94.91 &	81.35 & 50.32\\
            {\synthcat}10& $10\times$\textsc{resnet152}      & 80.18 &	95.38 &	82.39 & 52.73 \\
        \bottomrule
    \end{tabular}
    
\end{table}

\subsection{Fine-tuning experiments}
For reference, Table \ref{tab:imagenet_sl_ft_2ft} provides numerical results for the fine-tuning experiments of Figure \ref{fig:imagenet_sl_ft_2ft}.
\begin{table}[ht]
    \caption{Supervised transfer learning by either normal fine-tuning or proposed two-stage fine-tuning. Various representations are pre-trained on \textsc{ImageNet} and then fine-tuned or two-stage fine-tuned on \textsc{Cifar10}, \textsc{Cifar100}, \textsc{Inat18} tasks. }
    \bigskip
    \label{tab:imagenet_sl_ft_2ft}
    \centering

    \begin{tabular}{cc c | cc c | ccc}
    \toprule
                 &              &        & \multicolumn{3}{c|}{fine-tuning} & \multicolumn{3}{c}{two-stage fine-tuning} \\
         method  & architecture          &params &  \textsc{Cifar10}  &  \textsc{Cifar100} &  \textsc{Inat18} &  \textsc{Cifar10} &  \textsc{Cifar100} &  \textsc{Inat18} \\
         
         \midrule
             ERM & \textsc{resnet50}     & 23.5M & 97.54	& 85.58	& 64.19 &  - & 	- & - \\
        \midrule
             ERM & \textsc{resnet50w2}   & 93.9M & 97.76	& 87.13	& 66.72 &  - & 	- & - \\
             ERM & \textsc{resnet50w4}   & 375M  & 97.88	& 87.95 & 66.99 &  - & 	- & -\\
             ERM & \textsc{2$\times$resnet50}    & 47M   & 97.39	& 85.77	& 62.57 &  - & 	- & - \\
             ERM & \textsc{4$\times$resnet50}    & 94M   & 97.38	& 85.56	& 61.58 &  - & 	- & - \\
             \midrule
             {\synthcat}2& \textsc{2$\times$resnet50}     & 47M   & 97.56	& 86.04	& 64.49 & 97.87	& 87.07	& 66.96\\
             {\synthcat}4& \textsc{4$\times$resnet50}     & 94M   & 97.53	& 86.54	& 64.54 & 98.14	& 88.00	& 68.42\\
             {\synthcat}5& \textsc{5$\times$resnet50}     & 118M  & 97.57	& 86.46	& 64.86 & 98.19	& 88.11	& 68.48\\
            {\synthcat}10& \textsc{10$\times$resnet50}     & 235M  & 97.19	& 86.65 & 64.39 & 98.17	& 88.50	& 69.07\\
            \midrule
            {\synthdistill}5& \textsc{resnet50}     & 23.5M & 97.07	& 85.31	& 64.17 &  - & 	- & -\\
        \bottomrule
    \end{tabular}
    
\end{table}

\subsection{Vision transformer Experiment settings}
{For all vision transformer experiments, we keep the input image resolution at 384 $\times$ 384 and follow a similar protocol as appendix \ref{apx:imagenet_sl_settings}. Specifically, we use a weight decay=5e-4 and a batch size=256 for linear probing, a weight decay=0 and a batch size=512 (following the \citet{dosovitskiy2020image} settings) for fine-tuning and two-stage fine-tuning. Following \citet{dosovitskiy2020image}, all input images are normalized by $mean=(0.5, 0.5, 0.5), std=(0.5, 0.5, 0.5)$.}

\section{Self-supervised transfer learning}
\label{apx:ssl}
\subsection{SWAV on \textsc{ImageNet}}
SWAV is a contrastive self-supervised learning algorithm proposed by \citet{caron2020unsupervised}. We train \textsc{resnet50} on \textsc{ImageNet}\footnote{\url{https://github.com/facebookresearch/swav/blob/main/scripts/swav_400ep_pretrain.sh}} by the SWAV algorithm four times, which gives us four pretrained \textsc{resnet50} models. As to the rest four SWAV pre-trained models in this work, we  use the public available \textsc{resnet50}\footnote{\url{https://dl.fbaipublicfiles.com/deepcluster/swav_400ep_pretrain.pth.tar}}, \textsc{resnet50w2}\footnote{\url{https://dl.fbaipublicfiles.com/deepcluster/swav_RN50w2_400ep_pretrain.pth.tar}}, \textsc{resnet50w4}\footnote{\url{https://dl.fbaipublicfiles.com/deepcluster/swav_RN50w4_400ep_pretrain.pth.tar}}, and \textsc{resnet50w5}\footnote{\url{https://dl.fbaipublicfiles.com/deepcluster/swav_RN50w5_400ep_pretrain.pth.tar}} checkpoints.

\paragraph{Linear probing:} Following the settings in \citet{goyal2022vision}, the linear probing experiments (on \textsc{ImageNet}, \textsc{Inat18}, \textsc{Cifar100}, \textsc{Cifar10}) add a \textsc{BatchNorm} layer before the last-layer linear classifier to reduce the hyper-parameter tuning difficulty. The learning rate is initialized to 0.01 and multiplied by 0.1 every 8 epochs. Then train these linear probing tasks for 28 epochs by SGD Nesterov optimizer with momentum=0.9. We search L2 weight decay from $\{5e-4\}$, $\{5e-4, 1e-3, 5e-3, 1e-2\}$, and $\{1e-6, 1e-5, 1e-4\}$ for \textsc{ImageNet}, \textsc{Cifar}, and \textsc{Inat18} tasks, respectively. 

\paragraph{Fine-tuning:} 
\begin{itemize}
    \item \textbf{\textsc{ImageNet}}: Inspired by the semi-supervised \textsc{ImageNet} fine-tuning settings in \citet{caron2020unsupervised}, we attach a randomly initialized last-layer classifier on top of the SSL learned representation. Then fine-tune all parameters, using a SGD optimizer with momentum=0.9 and L2 weight decay=0. Low-layers representation and last-layer classifier use different initial learning rates of 0.01 and 0.2, respectively. The learning rate is multiplied by 0.2 at $12^{th}$ and $16^{th}$ epochs. We train 20 epochs for networks: \textsc{resnet50}, \textsc{resnet50w2}, \textsc{resnet50w4}. We further search training epochs from $\{10, 20\}$ for the wide network (due to overfitting), \textsc{resnet50w5} and then select the best one with 10 training epochs. 
    \item \textbf{\textsc{Cifar10, Cifar100, Inat18}}: Same as the fine-tuning settings in supervised transfer learning in Appendix \ref{apx:imagenet_sl_settings}.

\end{itemize}

\paragraph{Two-stage fine-tuning:} 
\begin{itemize}
    \item  \textbf{\textsc{ImageNet}}: Similar to the two-stage fine-tuning settings in supervised transfer learning, we initialize the last-layer classifier $w$ by concatenation and then train 1 epoch with learning rate=0.001, L2 weight decay=0. 
    \item \textbf{\textsc{Cifar10, Cifar100, Inat18}}: For \textsc{Cifar10, Cifar100}, we use same two-stage fine-tuning settings as in supervised transfer learning in Appendix \ref{apx:imagenet_sl_settings}. For \textsc{Inat18}, we attach a \textsc{BatchNorm} layer before the last-layer linear classifier to reduce the training difficulty. Note that \textsc{BatchNorm} + a linear classifier is still a linear classifier during inference. Following the linear probing protocol, we train the \textsc{BatchNorm} and linear layers by a SGD optimizer with momentum=0.9, initial learning rate=0.01, and a 0.2 learning rate decay at $12^{th}$ and $16^{th}$ epochs. As to L2 weight decay,  we use the same searching space as in the fine-tuning. 
\end{itemize}

\subsubsection{Additional results}
\label{apx:swav_additional_exp}
Beside the SWAV \textsc{ImageNet} fine-tuning experiments in Figure \ref{fig:ssl_tf}, Figure \ref{fig:swav_ft_full} provides additional SWAV fine-tuning / two-stage fine-tuning results on  \textsc{naturalist18}, \textsc{Cifar100}, and \textsc{Cifar10} tasks. We give a {\plotlabel{[init]cat}} curve on the \textsc{ImageNet} task, but omit the curves on other tasks (\textsc{naturalist18}, \textsc{Cifar100}, and \textsc{Cifar10}) because they are computationally costly. 
\begin{figure}[ht]
    \centering
    \includegraphics[width=\textwidth]{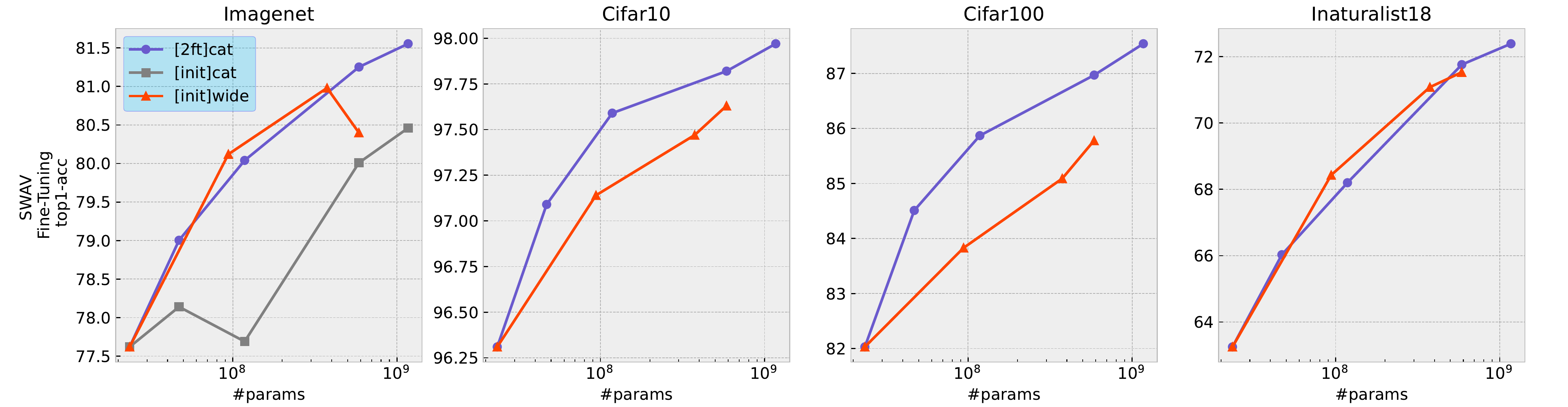}
    \caption{Fine-tuning performance of SWAV on \textsc{ImageNet}, \textsc{naturalist18}, \textsc{Cifar100}, and \textsc{Cifar10} tasks. SWAV is trained on unlabeled \textsc{ImageNet}. {\plotlabel{[2ft]cat}} and {\plotlabel{[init]cat} indicate our two-stage fine-tuning strategy and the normal fine-tuning strategy on $n$ concatenated networks. {\plotlabel{[init]wide}} refers to the normal fine-tuning strategy on wide networks, i.e. \textsc{resnet50}, \textsc{resnet50w2}, \textsc{resnet50w4}, and \textsc{resnet50w5}.    }}
    \label{fig:swav_ft_full}
\end{figure}

\subsection{SEER on \textsc{Instagram1B}}
SEER \citep{goyal2022vision} trains large \textsc{regnet\{32gf, 64gf, 128gf, 256gf, 10B\}} architectures on the \textsc{Instagram1B} dataset with 1 billion Instagram images, using the SWAV contrastive self-supervised learning algorithm.

\paragraph{Linear Probing:} Same as the linear probing settings in SWAV. 

\paragraph{Fine-tuning:} We use SEER checkpoints\footnote{\url{https://github.com/facebookresearch/vissl/tree/main/projects/SEER}} fine-tuned on \textsc{ImageNet} with $384 \times 384$ resolutions. It is fine-tuned on \textsc{ImageNet} for 15 epochs using SGD momentum 0.9, weight decay 1e-4, learning rate 0.04 and batch size 256. The learning rate is multiplied by 0.1 at $8^{th}$ and $12^{th}$ epochs.

\paragraph{Two-stage Fine-tuning:} We keep L2 weight decay 1e-4 the same as fine-tuning. Then keep the other settings the same as in SWAV.

\subsection{Additional experiment: \textsc{SimSiam} on \textsc{Cifar}}
\label{apx:simsiam_cifar}
\textsc{SimSiam} \cite{chen2020simsiam} is a non-contrastive self-supervised learning algorithm. In this section, we pre-train the networks using \textsc{SimSiam} on \textsc{Cifar10}, then transfer the learned representation by linear probing to   \textsc{Cifar10}, \textsc{Cifar100},  \textsc{Cifar10} with 1\% training examples, and \textsc{Cifar100} with 10\% training examples. 

\paragraph{\textsc{SimSiam} pre-training} Following \citet{chen2020simsiam} we pre-train \textsc{resnet18}, \textsc{resnet18w2}, \textsc{resnet18w4}, \textsc{2resnet18}, and \textsc{4resnet18} on \textsc{Cifar10} ($32 \times 32$ resolution) by \textsc{SimSiam} for $800$ epochs, using a SGD optimizer with momentum = $0.9$, initial learning rate = $0.06$, batch size = $512$, L2 weight decay = $5e-4$, and cosine learning rate scheduler. The data augmentations include \textsc{RandomResizedCrop} (crop scale in $[0.2, 1]$), \textsc{RandomHorizontalFlip}, \textsc{RandomGrayScale} ($p=0.2$), and a random applied \textsc{ColorJitter} ($0.4, 0.4, 0.4, 0.1$) with probability $0.8$. All images are normalized by $mean=(0.4914, 0.4822, 0.4465), std=(0.2023, 0.1994, 0.2010)$ before training.

\paragraph{\textsc{\synthdistill}} Since self-supervised learning tasks don't contain target labels as supervised learning, we apply knowledge distillation on representation directly. Specifically, we set $v_1, \dots v_n$ in Figure \ref{fig:distill} as Identity matrices. Then we distill $[\phi_1, \dots, \phi_n]$ into $\Phi$ by use a cosine loss: 

\begin{equation}
    \min_{\Phi, w_0, \dots, w_n}\sum_{i=0}^n\sum_x\bigg[  1-\cos\Big({\phi_i}(x)~, ~ w_i \circ \Phi(x) \Big)\bigg]
\end{equation}

\paragraph{Linear Probing:} Following again the settings of \citet{goyal2022vision}, the linear probing experiments
(on \textsc{Cifar100}, \textsc{Cifar10}, \textsc{Cifar100(1\%)} with 10\% training data, and \textsc{Cifar10(1\%)} with 1\% training data) adds a \textsc{BatchNorm} layer before the last-layer
linear classifier to reduce the hyper-parameter tuning difficulty. We use batch size = 256 for \textsc{Cifar100} and \textsc{Cifar10}, use batch size = 32 for corresponding sampled (10\%/1\%) version. Then we search initial learning rate from $\{0.1, 0.01\}$, L2 weight decay from \{1e-4, 5e-4, 1e-3, 5e-3\}. The learning rate is multiplied by 0.1 every 8 epochs during the total 28 training epochs. As to the optimizer, all experiments use a SGD Nesterov optimizer with momentum=0.9. 

\paragraph{Results}
Table \ref{tab:simsiam} shows the linear probing accuracy of \textsc{SimSiam} learned representation on various datasets and architectures. When linear probing on the same \textsc{Cifar10} dataset as training, the {\synthcat}$n$ method performs slightly better than width architectures (e.g. \textsc{resnet18w2} and \textsc{resnet18w4}). When comparing them on the  \textsc{Cifar100} dataset (OOD), however, {\synthcat}$n$ exceeds width architectures.

\begin{table}[ht]
    \centering
    \caption{Linear probing accuracy on \textsc{Cifar100}, \textsc{Cifar10}, \textsc{Cifar100(1\%)}, and \textsc{Cifar10(10\%)} tasks. The representation is learned on \textsc{Cifar10} by \textsc{SimSiam} algorithm. {\synthcat}$n$ concatenates $n$ learned representation before linear probing. {\synthdistill}$n$ distills $n$ learned representation into $\textsc{resnet18}$ before linear probing. \textsc{resnet18w}$n$ contains around $n^2$ parameters as {\textsc{resnet18}}.}
    \bigskip

    \begin{tabular}{cc|cc|cc}
    \toprule
    &  &\multicolumn{2}{c|}{Linear Probing (ID)} & \multicolumn{2}{c}{Linear Probing (OOD) } \\
    method & architecture & \textsc{Cifar10} & \textsc{Cifar10(1\%)} & \textsc{Cifar100} &  \textsc{Cifar100(10\%)} \\
    \midrule
    \textsc{SimSiam}  &\textsc{resnet18}       &  91.88   &87.60   &  55.29   &    42.93 \\
    \midrule
    \textsc{SimSiam}  &\textsc{resnet18w2}     &  92.88   & 88.95   &  59.41   &   45.39\\
    \textsc{SimSiam}  &\textsc{resnet18w4}     &  93.50   &90.45   &  59.28   &    44.98\\
    \textsc{SimSiam}  &\textsc{2resnet18}      &  91.62   &87.14   &  55.67   &    43.07\\
    \textsc{SimSiam}  &\textsc{4resnet18}      &  92.54   &85.65   &  64.42   &    49.65\\
    \midrule
    {\synthcat}2 & 2$\times$\textsc{resnet18}      &  92.94     & 88.32   & 59.40   &  46.06\\
    {\synthcat}4 & 4$\times$\textsc{resnet18}      &  93.42   & 88.81   &  63.06   &   47.48\\
    {\synthcat}5 & 5$\times$\textsc{resnet18}      &  93.67   & 88.78   & 63.71   &    48.31\\
   {\synthcat}10 & 10$\times$\textsc{resnet18}      &  93.75   & 88.65   &   66.19   &  49.90\\
   \midrule
{\synthdistill}2 & 2$\times$\textsc{resnet18}       &93.04	&88.59	&59.65	&45.10 \\
{\synthdistill}5 & 5$\times$\textsc{resnet18}       &93.02	&88.56	&60.79	&46.41 \\
{\synthdistill}10 & 10$\times$\textsc{resnet18}      &93.11	&88.72	&61.35	&46.75 \\
    \bottomrule
    \end{tabular}
    
    \label{tab:simsiam}
\end{table}

\subsection{Numerical results}
For reference, Tables \ref{tab:swav} and \ref{tab:seer} provide the numerical results for the linear probing, fine-tuning, and two-stage fine-tuning plots of Figure \ref{fig:ssl_tf}.
\begin{table}[ht]
    \centering
    \caption{Linear probing, fine-tuning, and two-stage fine-tuning performance of SWAV pre-trained representation and corresponding {\synthcat}$n$ representations. }
    \label{tab:swav}
    \bigskip

    \begin{tabular}{ccc|cccc|c|c}
\toprule
                &            &        &     \multicolumn{4}{c|}{ \footnotesize linear-probing}     & \footnotesize fine-tuning & \footnotesize two-stage ft \\   
\footnotesize method     & \footnotesize architecture   & \footnotesize params & \footnotesize \textsc{ ImageNet} & \footnotesize \textsc{Cifar10} & \footnotesize \textsc{Cifar100} & \footnotesize \textsc{Inat18} & \footnotesize \textsc{Imagenet} & \footnotesize \textsc{Imagenet} \\
\midrule
SWAV             & \textsc{resnet50}   & 23.5M & 74.30 & 91.83 & 76.85 & 42.35 & 77.62 & - \\
SWAV              & \textsc{resnet50w2} & 93.9M & 77.31 & 93.97 & 79.49 & 47.55 & 80.12 & - \\
SWAV              & \textsc{resnet50w4} & 375M & 77.48  & 94.29 &	80.51	& 44.13 & 80.98 & - \\
SWAV              &  \textsc{resnet50w5} & 586M  & 78.23 & 94.84 & 81.54 & 48.11 & 80.40 & - \\
\midrule
{\synthcat}2          & - & 47M  & 76.01 & 93.48 & 78.91 & 45.57 & 78.14 & 79.00 \\
{\synthcat}5          & - & 118M & 77.43 & 94.62 & 81.11 & 49.12 &  77.69     & 80.04 \\
{\synthcat}7    & - & 587M & 78.72 & 95.59 & 82.71 & 49.68 & 80.05  & 81.25 \\
{\synthcat}9 & - & 1170M& 78.89 & 95.76 & 83.16 & 50.61 & 80.46 & 81.55 \\
\bottomrule
    \end{tabular}
    
\end{table}

\begin{table}[ht]
    \centering
    \caption{Linear probing, fine-tuning, and two-stage fine-tuning performance of SEER pre-trained representation and corresponding {\synthcat}$n$ representations. }
    \bigskip
    \label{tab:seer}

    \begin{tabular}{ccc|cccc|c|c}
\toprule
                &            &        &     \multicolumn{4}{c|}{ \footnotesize linear-probing}     & \footnotesize fine-tuning & \footnotesize two-stage ft \\   
\footnotesize method     & \footnotesize architecture   & \footnotesize params & \footnotesize \textsc{ ImageNet} & \footnotesize \textsc{Cifar10} & \footnotesize \textsc{Cifar100} & \footnotesize \textsc{Inat18} & \footnotesize \makecell{\textsc{Imagenet} \\ (384px)} & \footnotesize \makecell{\textsc{Imagenet} \\ (384px)} \\
\midrule
SEER              & \textsc{regnet32gf}   & 141M & 73.4   &   89.94   &   71.53   & 39.10     &   83.4   &   - \\
SEER              & \textsc{regnet64gf}   & 276M & 74.9   &   90.90   &   73.78   &  42.69    &   84.0   &   - \\ 
SEER              & \textsc{regnet128gf}  & 637M  & 75.9   &   91.37   &   74.75   &  43.51    &   84.5   &   - \\
SEER              & \textsc{regnet256gf}  & 1270M  & 77.5   &   92.16   &   74.93   &  46.91    &   85.2   &   - \\
\midrule
{\synthcat}2             & - & 418M  & 76.0   &   92.16   &   75.65   &  45.36    &  -    &   84.5 \\
{\synthcat}3          & - & 1060M  & 77.3   &   93.15   &   77.26   &  47.18    &  -    &   85.1 \\
{\synthcat}4      & - & 2330M & 78.3   &   93.59   &   78.80   & 48.68     &    -  &   85.5 \\
\bottomrule
    \end{tabular}
    
\end{table}

\section{meta-learning / few-shots learning}
\label{apx:meta-learning}
\subsection{Datasets}
\textbf{\textsc{Cub}} \citep{wah2011caltech} dataset contains $11,788$ images of 200 birds classes, 100 classes ($5,994$ images) for training and 100 classes ($5,794$ images) for testing. 

\textbf{\textsc{MiniImageNet}} \citep{matchingnet} dataset contains $60,000$ images of 100 classes with 600 images per class, 64 classes for training, 36 classes for testing. 

\subsection{\textsc{Baseline} and \textsc{Baseline++} experiment Settings}
\label{apx:meta_baseline_pretrain}
For \textsc{Baseline} and \textsc{Baseline++} experiments, following \citet{closelookatfewshot},  we use \textsc{RandomSizedCrop}, \textsc{ImageJitter}(0.4, 0.4, 0.4), and \textsc{HorizontalFlip} augmentations, as well as a image normalization $mean=(0.485, 0.456, 0.406)$, $std=(0.229, 0.224, 0.225)$. Then use an \textsc{Adam} optimizer with learning rate = 0.001, batch size = 16, input image size = $224\times224$. Finally, train \textsc{resnet18} on \textsc{Cub} and \textsc{MiniImageNet} datasets for 200 and 400 epochs, respectively. We further tune L2 weight decay from \{0, 1e-5, 1e-4, 1e-3, 1e-2\} and choose 1e-4 for \textsc{Cub}, 1e-5 for \textsc{MiniImageNet} experiments. Compared with the \textsc{Baseline} and \textsc{Baseline++} performance reported by \citet{closelookatfewshot} (table A5), this L2 weight decay tuning process provides $\sim5\%$ and $\sim1\%$ improvement on \textsc{miniImagenet} 5way-1shot and  5way-5shot, respectively. In this work, we use this stronger setting in baseline methods. 

As to the few-shots learning evaluation, following \citet{closelookatfewshot}, we scale images by a factor of 1.15, \textsc{CenterCrop}, and normalization. Then randomly sample 1 or 5 images from 5 random classes from the test set (5way-1shot and 5way-5shot). Finally, train a linear classifier on top of the learned representation with a SGD optimizer, momentum = 0.9, dampening = 0.9, learning rate = 0.1, L2 weight decay = 1e-3, batch size = 4, and epochs = 100. We take the average of 600 such evaluation processes as the test score.

 The \textsc{Baseline} and \textsc{Baseline++} results in Figure \ref{fig:few_shots} report the mean of five runs with different training and evaluating seeds.

\paragraph{Implementation details of the cosine classifier}
Here we summarize the technical details of the cosine classifier implementation used in this work which follows \citet{closelookatfewshot}\footnote{%
\url{https://github.com/wyharveychen/CloserLookFewShot/blob/master/backbone.py##L22}}. 

Denote the representation vector as $z$. The cosine classifier calculates the $i^{th}$ element of logits by: 
\begin{equation}
     h_i = g_i \frac{\left<u_i,z\right>}{||u_i||||z||}
\end{equation}
where $u_i$ is a vector with the same dimension of $z$, $g_i$ is a scalar, $h_i$ is $i^{th}$ element of logits $h$.

Then minimize the cross entropy loss between the target label $y$ and softmax output $s(h)$ by updating $w$ and $g$: $\min_{w,g} \mathcal{L}_{ce}(y, s(h))$.

\subsection{{\synthcat} and {\synthdistill} experiment settings}

For {\synthcat}, we concatenate $n$ representation separately trained by either \textsc{Baseline} or \textsc{Baseline++} as the settings above.  For {\synthdistill}, we use the same multi-head architecture as figure \ref{fig:distill} together with a cross-entropy loss function:
\begin{equation}
\label{eq:ce_kl_distill}
    \min_{\Phi, w_0, \dots, w_n}\sum_{i=0}^n\sum_x\bigg[ (1-\alpha) \mathcal{L}_{ce}\Big(s\big(w_i \circ \Phi(x)\big), y \Big) + \alpha\tau^2\mathcal{L}_{kl}\Big(s_{\tau}\big({v_i} \circ {\phi_i}(x)\big) ~||~ w_i \circ \Phi(x) \Big)\bigg]
\end{equation}, 
where $\mathcal{L}_{ce}$ indicates a cross-entropy loss, $\alpha$ is a trade-off parameter between cross-entropy loss and kl-divergence loss. We set L2 weight decay = 0,  $\tau=10$, search $\alpha \in \{0.8, 0.9, 1\}$, and keep the other hyper-parameters as Appendix \ref{apx:meta_baseline_pretrain}. We find the impact of $\alpha$ is limited in both \textsc{cub} ($\leq 1\%$) and \textsc{MiniImageNet} ($ \leq 0.5\%$) tasks. 

\subsection{Snapshots experiment settings}
In this section, we apply {\synthcat} and {\synthdistill} on 5 snapshots sampled from one training episode (called \textsc{cat5-s} and \textsc{distill5-s}, respectively). We train $\textsc{cub}$ and $\textsc{MiniImageNet}$ respectively for 1000 and 1200 epochs by naive SGD optimizer with a relevant large learning rate 0.8. Then we sample 5 snapshots, $\{200^{th}, 400^{th}, 600^{th}, 800^{th}, 1000^{th}\}$ and  $\{400^{th}, 600^{th}, 800^{th}, 1000^{th}, 1200^{th}\}$, for $\textsc{cub}$ and $\textsc{MiniImageNet}$, respectively. The other hyper-parameters are the same as Appendix \ref{apx:meta_baseline_pretrain}.

\subsection{More experimental results}
Table \ref{tab:few_shot_synt_cat} provides the exact number in Figure \ref{fig:few_shots}, as well as additional {\synthcat}$n$ and {\synthdistill}$n$ few-shots learning results with a linear classifier (The orange and gray bars in figure \ref{fig:few_shots} report the few-shots learning performance with a cosine classifier). 

{Table \ref{tab:few_shot_synt_cat_snapshot} provides more \synthcat5\textsc{-s} and \synthdistill5\textsc{-s} results with either a linear classifier or a cosine-based classifier.}
\begin{table}[t]
    \caption{Few-shots learning performance on \textsc{cub} and \textsc{miniImagenet}. The \textsc{cat5-s} and \textsc{distill5-s} results were obtained using five snapshots taken during a single training episode with a relatively high step size (0.8, SGD). The best snapshot performances are also reported. Standard deviations over five repeats are reported.
    }
    \label{tab:few_shot_synt_cat_snapshot}
    \par\bigskip\centering
    \begin{tabular}{ccc |cc |cc}
  \toprule
 &           &            & \multicolumn{2}{c|}{\textsc{cub}}         &     \multicolumn{2}{c}{\textsc{miniImagenet} } \\
                & architecture & classifier       &   5way 1shot    &    5way 5shot &   5way 1shot    &    5way 5shot    \\
    \midrule
    best snapshot & \textsc{resnet18} & linear & 59.70$\pm$1.38 &	81.35$\pm$0.79 & 52.79$\pm$0.92 & 75.18$\pm$0.57 \\ 
    \synthcat5\textsc{-s} & \textsc{5$\times$resnet18} & linear& 72.62$\pm$0.98 & 86.56$\pm$0.82 & 61.91$\pm$0.37 &	81.06$\pm$0.14 \\
\synthdistill5\textsc{-s} & \textsc{resnet18} & linear& 68.4$\pm$0.5 & 87.2$\pm$0.4 &
    59.9$\pm$0.5 & 80.8$\pm$0.4  \\
    \midrule
    best snapshot & \textsc{resnet18} & cosine & 65.59$\pm$0.87 & 81.81$\pm$0.50 & 55.67$\pm$0.48 &	75.48$\pm$0.46 \\
    \synthcat5\textsc{-s} & \textsc{5$\times$resnet18} & cosine & 73.66$\pm$0.82 & 87.25$\pm$0.77 & 62.94$\pm$0.51 &	81.05$\pm$0.16 \\
\synthdistill5\textsc{-s} & \textsc{resnet18} & cosine &
     75.2$\pm$0.8 & 88.6$\pm$0.4 & 62.0$\pm$0.5 & 81.0$\pm$0.3 \\
    \bottomrule
    \end{tabular}
\end{table}
\begin{table}[ht]
    \centering
    \caption{Few-shot learning performance on \textsc{cub} and \textsc{miniImagenet} dataset with either a linear classifier or cosine-distance based classifier. Standard deviations over five repeats are reported. }
    \bigskip

    \label{tab:few_shot_synt_cat}
    \begin{tabular}{ccc |cc |cc}
  \toprule
 &           &            & \multicolumn{2}{c|}{\textsc{cub}}         &     \multicolumn{2}{c}{\textsc{miniImagenet} } \\
                & architecture & classifier       &   5way 1shot    &    5way 5shot &   5way 1shot    &    5way 5shot    \\
\midrule  
supervised 	    &  \textsc{resnet18} & linear &          63.37$\pm$1.66 & 83.47$\pm$1.23	&	55.20$\pm$0.68 & 76.52$\pm$0.42	\\
{\synthcat}2	&  $2\times$\textsc{resnet18} &   linear &66.25$\pm$0.85 & 85.50$\pm$0.34	&	57.30$\pm$0.31 & 78.42$\pm$0.17	\\
{\synthcat}5	&  $5\times$\textsc{resnet18} &   linear &67.00$\pm$0.18 & 86.80$\pm$0.10	&	58.40$\pm$0.25 & 79.59$\pm$0.17	\\
{\synthdistill}2&  \textsc{resnet18} &   linear&69.93$\pm$0.74 & 87.72$\pm$0.31	&	58.99$\pm$0.32 & 79.73$\pm$0.21	\\
{\synthdistill}5&   \textsc{resnet18} &  linear&70.99$\pm$0.31 & 88.52$\pm$0.14	&	59.66$\pm$0.59 & 80.53$\pm$0.27	\\
\midrule
supervised 	    &  \textsc{resnet18} &   cosine&69.19$\pm$0.88 & 84.41$\pm$0.49	&	57.47$\pm$0.45 & 76.47$\pm$0.27	\\
{\synthcat}2	&  $2\times$\textsc{resnet18} &   cosine&72.87$\pm$0.43 & 86.82$\pm$0.17	&	60.69$\pm$0.24 & 79.29$\pm$0.23	\\
{\synthcat}5	&  $5\times$\textsc{resnet18} &   cosine&76.23$\pm$0.55 & 88.87$\pm$0.40	&	63.63$\pm$0.23 & 81.22$\pm$0.17	\\
{\synthdistill}2&  \textsc{resnet18} &   cosine &  74.81$\pm$0.45 & 88.14$\pm$0.40	&	61.95$\pm$0.11 & 80.79$\pm$0.26	\\
{\synthdistill}5&  \textsc{resnet18} &   cosine & 76.20$\pm$0.39 & 89.18$\pm$0.24	&	62.89$\pm$0.38 & 81.49$\pm$0.26	\\
\bottomrule
    \end{tabular}
\end{table}

\subsection{Comparison with conditional Meta-learning approaches}
In order to address heterogeneous distributions over tasks, the conditional meta-Learning approaches \citet{wang2020structured,denevi2022conditional,rusu2018meta} adapt a part of model parameters conditioning on the target task, while freeze the other model parameters that are pre-trained as a feature extractor. 

The results presented in \citet{wang2020structured,denevi2022conditional,rusu2018meta} already allow us to make some elementary comparisons: \citet{denevi2022conditional} is derived from \citet{wang2020structured}. In practice, \citet{wang2020structured} reuses the pre-trained frozen feature extractor (\textsc{WRN-28-10}) from \citet{rusu2018meta}. Table \ref{tab:miniimagenet_conditional_metalearning} below shows the performance of these conditional meta-learning methods and our \textsc{Distill5} on the \textsc{miniImagenet} few-shot learning task. The first 3 rows are copied from \citet{wang2020structured} (marked by *). Despite the fact that the backbone in \citet{wang2020structured,rusu2018meta} (\textsc{WRN-28-10}) is wider and deeper than the backbone (\textsc{resnet18}) used in our paper, \textsc{Distill5} still outperforms both \citet{wang2020structured} and \citet{rusu2018meta}. Other relevant details are summarized in table \ref{tab:backbone_pretraining_details}.

If our goal were to present state-of-the-art results exploiting diverse features, a more systematic comparison would be needed. however it is not clear that these results say a lot about how optimization constructs and (prematurely) prunes features. The conditional meta-learning addresses an orthogonal problem but does not seem to fix the premature feature pruning issue. Please not that the message of our paper is that a single optimization run — which is what most people are doing these days - prematurely prunes its representations, missing opportunities to produce the richer representations that benefit out-of-distribution scenarios.

\begin{table}[h]
    \centering
    \begin{tabular}{c|cc}
    \toprule
      &  miniImageNet 5way-1shots &	miniImageNet 5way-5shots \\
      \midrule
LEO \cite{rusu2018meta} &	61.76$\pm$0.08*	&77.59$\pm$0.12* \\
LEO(local) \cite{rusu2018meta} &	60.37$\pm$0.74*	&75.36$\pm$0.44* \\
TASML \cite{wang2020structured}&	62.04$\pm$0.52*	&78.22$\pm$0.47* \\
Distill5 (our)&	62.89$\pm$0.38&	81.49$\pm$0.26 \\
\bottomrule
    \end{tabular}
    \caption{\textsc{miniImageNet} few-shots learning comparison between \textsc{Distill5} and conditional meta-learning approaches. The first three rows are copied from corresponding papers (marked by *).}
    \label{tab:miniimagenet_conditional_metalearning}
\end{table}

\begin{table}[h]
    \centering
    \begin{tabular}{c|cc}
    \toprule
& Our backbone  &	LEO backbone \cite{rusu2018meta,wang2020structured} \\
\midrule
Architecture &	\textsc{resnet18} &	\textsc{WRN-28-10} \\ 
Parameters	& 11.4M &	36.5M \\
L2 weight decay	& \checkmark	& \checkmark \\
Learning rate scheduler & 	$\times$	& \checkmark \\
Data augmentation (color)& 	\checkmark	& \checkmark \\
Data augmentation (scale)& 	\checkmark	& \checkmark \\
Data augmentation (deformation)& 	$\times$	& \checkmark \\
\bottomrule
    \end{tabular}
    \caption{Backbone pretraining details. Note that LEO only keeps the first 21 layers (21.7M parameters) after pretraining \textsc{WRN-28-10} (Wide residual network). But it is still twice the time larger than \textsc{resnet18}.}
    \label{tab:backbone_pretraining_details}
\end{table}

\section{Out-of-distribution learning}
\label{apx:ood}
Following \citet{zhang2022rich}, we use the \textsc{Camelyon17} \citep{koh2021wilds} task to showcase the $\synthcat$ and $\synthdistill$ constructed (rich) representation in out-of-distribution learning scenario. The first row of Table \ref{tab:camelyon17_synt_cat} is copied from \citet{zhang2022rich}. The rest results use a frozen pre-trained representation, either by concatenating $n$ ERM pre-trained representations ({\synthcat}$n$), distilling of $n$ ERM pre-trained representations ({\synthdistill}$n$), or \textsc{Bonsai} constructed representations(\textsc{Bonsai2}). Then train a linear classifier on top of the representation by vREx or ERM algorithms. 

For the vREx algorithm, we search the penalty weights from \{0.5, 1, 5, 10, 50, 100\}. For {\synthdistill}$n$ representations in the \textsc{Camelyon17} task, we follow Algorithm 2 in \citet{zhang2022rich}, but use a slightly different dataset balance trick in the loss function (\citet{zhang2022rich} Algorithm 2 line 13-14). We instead balance two kinds of examples: one shares the same predictions on all ERM pre-trained models, and one doesn't. We keep other settings to be the same as \citet{zhang2022rich}\footnote{\url{https://github.com/TjuJianyu/RFC}}.

\end{document}